%% file: arvix_version.tex
\documentclass{article}     
\usepackage[utf8]{inputenc}
\usepackage[left=3cm, right=3cm, top=3cm, bottom=3cm]{geometry}
\usepackage[colorlinks=true, allcolors=blue]{hyperref}
\usepackage{natbib}
\usepackage{bm,  mathrsfs, amsmath, amssymb, amsthm, amscd, amsfonts, graphicx,epsfig,latexsym,layout,fancyhdr, colonequals, enumerate, fontenc, graphics, verbatim, psfrag, lipsum, appendix}
\usepackage{algpseudocode}
\usepackage[all,cmtip]{xy}
\usepackage{tikz}
\usepackage{tikz-cd}
\usepackage{bbm}
\usepackage{multirow}
\usepackage{booktabs}                                   \usepackage{array}
\usepackage{makecell}
\usepackage{tikz-cd}
\usepackage{relsize} 
\usepackage{comment}
\usepackage{subcaption} 
\usepackage{soul}
\usepackage[ruled,vlined]{algorithm2e}
\usepackage{longtable}
\usepackage{booktabs}
\usepackage{authblk}
\usepackage{textcomp}   
\usepackage{hyperref}
\usepackage{adjustbox}

\newcommand{\keywords}[1]{
  \small
  \begin{quotation}
  \noindent\textbf{Keywords:} #1
  \end{quotation}
}
\newcommand{\expfig}[1]{\includegraphics[width=0.13\textwidth]{#1}}

\usetikzlibrary{matrix,arrows.meta,decorations.markings}
\definecolor{COrange}{HTML}{E69F00}
\definecolor{CRed}{HTML}{FF5050}
\definecolor{CBlue}{HTML}{0072B2}
\definecolor{CGreen}{HTML}{009E73}
\definecolor{CPink}{HTML}{CC79A7}
\definecolor{CDark}{HTML}{2B2B2B}
\tikzset{
  edge id/.style={
    line width=1.1pt,
    postaction={decorate,decoration={markings,
      mark=at position 0.32 with {\arrow{Latex[length=2.1mm,width=1.5mm]}},
      mark=at position 0.72 with {\arrow{Latex[length=2.1mm,width=1.5mm]}}}}
  },
  orbit/.style={line width=5.5pt,line cap=round,opacity=0.18},
  pt/.style={circle,inner sep=0pt,minimum size=4.5pt,fill},
  map/.style={-{Latex[length=2.8mm,width=1.9mm]},line width=0.95pt,CDark},
  dist/.style={<->,line width=0.95pt,CPink}
}

\newcommand{\D}{ D}
\newcommand{\ssf}{\mathrm{II}}
\newcommand{\M}{\mathcal M}
\newcommand{\N}{\mathcal N}
\newcommand{\X}{\mathcal X}
\newcommand{\RR}{\mathbb R}

\newcommand{\vol}{\mathrm{Vol}}

\newcommand{\x}{x}
\newcommand{\y}{y}

\newcommand\numberthis{\addtocounter{equation}{1}\tag{\theequation}}

\theoremstyle{plain}
\newtheorem{theorem}{Theorem}[section]
\newtheorem{corollary}[theorem]{Corollary}
\newtheorem{lemma}[theorem]{Lemma}
\newtheorem{proposition}[theorem]{Proposition}

\theoremstyle{definition}
\newtheorem{definition}[theorem]{Definition}
\newtheorem{example}[theorem]{Example}
\newtheorem{remark}[theorem]{Remark}

\title{\textbf{Group Invariant Spectral Embedding}}

\author[1,*,$\dagger$]{Yeari Vigder}
\author[2,*,$\dagger$]{Paulina Hoyos}
\author[3]{David Thong}
\author[3]{Joakim And\'en}
\author[2]{Joe Kileel}
\author[1]{Amit Moscovich}

\affil[1]{Department of Statistics and Operations Research, Tel Aviv University, Tel Aviv, Israel \protect\\
\href{mailto:vigderyeari@mail.tau.ac.il}{\texttt{vigderyeari@mail.tau.ac.il}} (Y. Vigder), \;
\href{mailto:mosco@tauex.tau.ac.il}{\texttt{mosco@tauex.tau.ac.il}} (A. Moscovich)}

\affil[2]{Department of Mathematics, University of Texas at Austin, Austin, TX, USA \protect\\
\href{mailto:paulinah@utexas.edu}{\texttt{paulinah@utexas.edu}} (P. Hoyos), \;
\href{mailto:jkileel@math.utexas.edu}{\texttt{jkileel@math.utexas.edu}} (J. Kileel)}

\affil[3]{Department of Mathematics, KTH Royal Institute of Technology, Stockholm, Sweden \protect\\
\href{mailto:dthong@kth.se}{\texttt{dthong@kth.se}} (D. Thong), \;
\href{mailto:janden@kth.se}{\texttt{janden@kth.se}} (J. And\'en)}

\affil[*]{Equal contribution}
\affil[$\dagger$]{Corresponding authors}

\date{October 8, 2026}

\begin{document}
\maketitle

\begin{abstract}
\noindent Spectral embedding methods are widely used for dimensionality reduction and clustering of high-dimensional datasets with intrinsic low-dimensional structures.
Although many datasets of practical interest exhibit invariance under symmetries such as rotations, standard spectral embedding methods do not account for this, treating symmetry-related data points as unrelated. 
Our approach to this problem is to incorporate the symmetries directly into the affinity kernels used for spectral embedding.
We analyze the case of a Riemannian data manifold $\M$ with symmetries given by a compact Lie group~$G$ and prove that, under suitable conditions, graph Laplacians constructed from three types of invariant kernels converge pointwise to explicit 
second-order differential operators on the quotient space $\M/G$.
Our analysis implies improved convergence rates, as the effective dimension drops according to the dimension of the group. 
We validate our approach on datasets with $\mathrm{SO}(2)$ or $\mathrm{SO}(3)$ symmetry, and show that $G$-invariant spectral embedding recovers the intrinsic geometry of the data, in contrast to standard spectral embedding, which fails to do so even in the limit of infinite data.
\end{abstract}




\section{Introduction}
Spectral embedding methods are powerful tools for analyzing high-dimensional data with intrinsic low-dimensional structure. 
Their basic operating principle is to 
construct a graph from pairwise affinities between data points and then use the eigenvectors of the graph Laplacian as low-dimensional representations. 
These methods are widely used for tasks such as dimensionality reduction, clustering, semi-supervised learning, and data denoising.
In this work, we consider datasets that exhibit invariance under known symmetry transformations.
For example, in single-particle cryo-electron microscopy (cryo-EM), molecular projection images are subject to random in-plane rotations.   
The 2D rotations are typically viewed as nuisance parameters, which  computational methods need to account for \citep{singer2020computational}.
As an example of a discrete symmetry, consider set-structured data such as unlabeled graphs, stored using an adjacency matrix.
Despite the vertices having no natural ordering, they are nonetheless assigned arbitrary row/column indices. However applying the same permutation on the rows and columns of the matrix results in an identical graph \citep{zaheer2017deep, maron2019invariant}.

Classical approaches for symmetry-aware machine learning fall into two main categories: data augmentation and invariant features. Data augmentation enlarges the dataset by including transformed copies of each data point; this is widely used in supervised learning and can provably reduce sample complexity \citep{chen2020group}. A second approach relies on handcrafted invariant features such as rotation-invariant descriptors for images \citep{Lowe2004} or permutation-invariant statistics for sets \citep{zaheer2017deep}. In the unsupervised setting neither approach is entirely satisfactory.
Data augmentation treats different transformations of one data point as separate observations, failing to account for the fact that they represent the same instance, and can also incur significant memory overhead.
While invariant features are effective in some settings, they do not provide a general recipe that transfers across different domains and symmetry types easily.

The central question we address in this paper is \emph{how should one perform spectral embedding on datasets with known symmetries?}
Rather than augmenting the data, we propose to ``augment" the affinity kernel, so that  observations are treated as representatives of an equivalence class.
 Mathematically, we seek to pass to the quotient manifold by using group invariant kernel functions.
Figure~\ref{fig:torus_quotient_intro} illustrates this on a simple example.

\begin{figure}
    \centering
    \resizebox{0.7\textwidth}{!}{\input{torus_orbit_intro_figure_tikz.tex}}
    \caption{An illustration of group orbits and the quotient manifold. The flat torus $\M=\mathbb{T}^2$ is drawn as a fundamental square with opposite edges identified. The group $G=\mathbb{S}^1$ acts by translations in the $a$-direction, so each horizontal orbit is one equivalence class: $x$ and $x'=g\cdot x$ represent the same quotient point, while $y$ lies on a distinct orbit. The quotient map $\pi\colon\M\to\N=\M/G$ collapses each horizontal orbit to a point on the $b$-coordinate.
    A $G$-invariant affinity kernel $K_G(x,y)$ is a real symmetric function that depends only on the $G$-orbits $[x]$ and $[y]$.}
    \label{fig:torus_quotient_intro}
\end{figure}
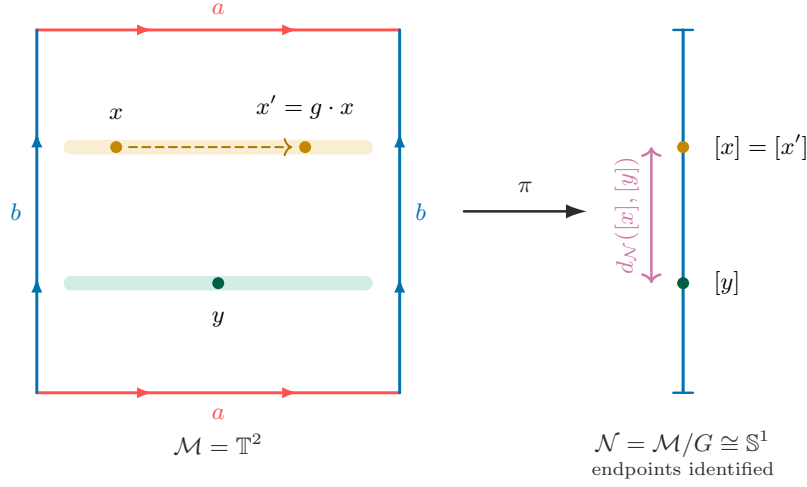

\subsection{Our Contributions}
We study three broadly applicable classes of group invariant affinity kernels with respect to a known group $G$:
(i)~minimization over $G$; (ii)~integration over $G$; and (iii)~$G$-invariant features mapping.
We analyze the continuous case where the data lies on a Riemannian manifold $\M$ and the symmetries are given by a compact Lie group $G$ that acts smoothly and freely on $\M$ via isometries.
Our main theoretical result (Theorem~\ref{thm: convergence of G-invariant graph Laplacian}) proves that graph Laplacians constructed from $G$-invariant kernels converge pointwise to explicit second-order differential operators on the quotient space $\N = \M/G$.  
Crucially, we establish an improvement in the convergence rate, from a rate of
\begin{align}
    O(\varepsilon)
    +
    O_P
    \left(
        n^{-1/2}\varepsilon^{-1/2-\mathrm{dim}(\M)/4}
    \right)
    \qquad\qquad\qquad\quad\ 
    \text{(see Eq.~\eqref{eq:LRW_convergence})}
\end{align}
for the standard graph Laplacian that converges to a  differential operator on $\M$, to
\begin{align}
    O(\varepsilon)
    +
    O_P
    \left(
        n^{-1/2}\varepsilon^{-1/2 -(\mathrm{dim}(\M) - \dim(G))/4}
    \right)    \qquad\qquad
    \text{(Corollary~\ref{cor:improved_convergence_rate})}
\end{align}
using $G$-invariant kernels that converge to operators on $\mathcal{M}/G$.
Here $\varepsilon$ is a bandwidth parameter and $n$ is the number of data points. 
This reflects an effective dimension reduction by exploiting symmetries in the data, similar to results on augmentation in supervised learning \citep{chen2020group}.
Further, Corollary~\ref{cor:eigen_descend} shows that when all orbits of the action have the same volume there is a bijection between the eigenfunctions on $\N$ and the  $G$-invariant eigenfunctions on $\M$.
In Section~\ref{sec:experiments}, we validate our framework on problems with $\mathrm{SO}(2)$ and $\mathrm{SO}(3)$ symmetry, demonstrating symmetry-aware embeddings and their better sample efficiency and interpretability compared to standard spectral  methods.

\subsection{Related Work}

\noindent \textbf{Manifold learning and spectral methods.}
Manifold learning has a long history in unsupervised data analysis.
Early examples are Isomap and LLE \citep{doi:10.1126/science.290.5500.2319,roweis2000nonlinear}, based on approximate geodesic distances and local linear approximations, respectively.
In this paper we consider spectral embeddings, two variants of which are Laplacian eigenmaps \citep{6789755} and diffusion maps \citep{COIFMAN20065}.
These methods use the eigenvectors of a graph Laplacian to obtain low-dimensional coordinates that reflect geometry intrinsic to the data. Convergence results for graph Laplacian methods were established in \citet{BELKIN20081289}, \citet{Hein2007}, and \citet{Singer2006}, with improved rates obtained  by \citet{calder2022improved} and \citet{Cheng2021EigenconvergenceOG}. The theory was extended to non-Euclidean affinities by \citet{Kileel_2021} and \citet{XuSinger2026}.
Spectral clustering leverages the same eigenvector embedding for unsupervised partitioning \citep{NIPS2001_801272ee,vonLuxburg2007tutorial}.
The consistency of spectral clustering was established by \citet{vonLuxburgBelkinBousquet2008}.
\\[-0.8em]

\noindent \textbf{Symmetry-aware neural networks.}
Machine learning increasingly exploits symmetries to improve efficiency and generalization. Data augmentation was analyzed by \citet{chen2020group} where the authors proved that in the presence of group symmetries, adding transformed copies of data points reduces the sample complexity of training. Theoretical analyses of the benefits of invariance for kernel methods and regression appeared in \citet{tahmasebi2023exact}.
Equivariant neural networks incorporate symmetry directly into neural network architectures using equivariant convolutional layers \citep{cohen2016group}, steerable CNNs \citep{cohen2017steerable} and/or spherical CNNs \citep{cohen2018spherical}. Tensor field networks \citep{thomas2018tensor} and SE(3)-Transformers \cite{fuchs2020se3} handle rotation and translation equivariance for 3D point clouds and E(n)-equivariant graph neural networks \citep{satorras2021en} provide a flexible framework for Euclidean symmetries.
For works related to permutation-symmetries of deep neural networks, see \citet{zaheer2017deep} and \citet{maron2019invariant}. A survey of geometric deep learning and equivariant architectures appears in \citet{bronstein2021geometric}.\\[-0.8em]

\noindent \textbf{Symmetry-aware spectral methods.}
The steerable graph Laplacian of \citet{SteerablePaper} is a symmetry-aware extension of the graph Laplacian tailored to datasets of 2D images with in-plane rotational symmetry.  That work constructs a rotation-equivariant operator by explicitly incorporating rotated copies in the affinity construction.
Closely related to our setting, \citet{ROSEN2024} introduced the $G$-invariant graph Laplacian for data lying on a manifold that is closed under the action of a known Lie group $G$. Its construction enforces invariance by incorporating distances across all pairs of points generated by the action of $G$ on the dataset. This amounts to implicit (or closed-form) data augmentation, since the construction is conceptually equivalent to the standard graph Laplacian applied to the infinitely many data points generated by the action of $G$.
Similar to the approach of equivariant neural networks, we incorporate group invariance directly into the learning method rather than relying on data augmentation but do so in the unsupervised setting of spectral embedding.
In addition to integration over $G$, we also consider minimization over $G$ and symmetry-invariant feature maps.
Recent works that use rotation-invariant metrics for 3D molecules include \citet{Diepeveen2024} and \citet{Zhang2024}.

\section{Background on Spectral Embedding}\label{sec:GL_methods}
Throughout this paper,  $\M \subseteq \RR^D$ denotes a $d$-dimensional connected compact Riemannian submanifold of $\RR^D$ without boundary. The data points $\X = \{ \x_1, \x_2, \dots, \x_n \}$ are sampled from $\M $ and given as vectors in the ambient space $\RR^{D}$. The geometric structure of the data manifold $\M$ is captured using a suitable affinity kernel $K \colon \RR^D \times \RR^D \to \RR$ which measures the pairwise similarity of data points.
Common choices of $K$ are the Gaussian kernel $K(x,y) = \exp(-\| x - y \|^2/ \varepsilon)$ and the $0/1$ kernel $K(x, y) = \mathbf{1}(\|x - y\| \leq \varepsilon)$, where $\| \cdot \|$ denotes the Euclidean norm in $\RR^D$ and $\varepsilon > 0$ is a bandwidth parameter.
See Appendix \ref{sec:notation} for a table of the notation used in the paper.

Let $W \in \RR^{n \times n}$ be a weight matrix defined by:
\begin{equation}\label{eq: weight matrix}
    W_{ij} := K(\x_i, \x_j).
\end{equation}
The matrix $W$ defines a weighted graph $(\X, E, W)$ with edges $E$ consisting of the pairs $\{\x_i, \x_j\}$ for which $W_{ij} >0$. We assume throughout that this graph is connected.
Let $D$ be the diagonal degree matrix with non-zero entries given by $D_{ii} := \sum_{j} W_{ij}.$
The random-walk normalized \emph{graph Laplacian} is the $n\times n$ matrix:
\begin{equation}\label{eq: RW graph Laplacian}
   L_{RW} \colonequals I - D^{-1}W,
\end{equation}
where $I$ is the identity matrix.
Since $D^{-1}W$ is a row-stochastic matrix, it can be interpreted as the transition probability matrix of a random walk on the graph $(\X, E, W)$. 
Since $\X$ is comprised of $n$ points we can identify functions $f \colon \X \to \RR$ with vectors in $\RR^n$ whose entries are indexed by the points in $\X$. The graph Laplacian defines a linear map on such functions, given by
\begin{equation}\label{eq: RW graph Laplacian acting on functions}
    (L_{RW} f)(\x_i)
    =
    \frac{1}{\sum_{j=1}^n K(\x_i, \x_j)}\sum_{j=1}^n K(\x_i, \x_j) \left( f(\x_i) - f(\x_j) \right).
\end{equation}
In words, at each point $\x_i \in \X$, the graph Laplacian averages the differences between the value $f(\x_i)$ and the neighbors' values, weighted by their affinities to $\x_i$.

The matrix $L_{RW}$
has real eigenvalues $0 =\lambda_0 < \lambda_1 \leq \cdots \leq \lambda_{n-1}$ with
corresponding eigenvectors 
$ \varphi_0, \varphi_1, \dots, \varphi_{n-1} \in \RR^n $
that form a basis for $\RR^n$, where $\varphi_0 = n^{-1/2}\mathbf{1}$ \citep{vonLuxburg2007tutorial}.
The smoothness of each eigenvector $\phi_i$ is measured by its Rayleigh quotient: $\lambda_i = (\varphi_i^T L_{RW} \varphi_i) / (\varphi_i^T \varphi_i)$. The eigenvectors corresponding to the $m$ smallest nonzero eigenvalues vary most smoothly over the graph, and therefore capture the coarsest geometric structure of the data. Mapping each data point $\x_i$ to its coordinates in this low-dimensional eigenvector basis, as detailed in Algorithm~\ref{alg:spectral embedding}, gives an embedding that reflects the intrinsic geometry of the data, bringing nearby points in the original space close together in the embedding.\footnote{Several names were given to this method and its variants over the years.
Key works include Laplacian eigenmaps \citep{6789755} and diffusion maps \citep{COIFMAN20065}.}

\begin{algorithm}
\caption{\textsc{Spectral Embedding}}
\label{alg:spectral embedding}
\DontPrintSemicolon
\KwIn{Data points $\x_1, \dots, \x_n \in \RR^D$, positive integer $m$, affinity kernel $K$.}
  \vspace{2pt}
\KwOut{Embedded data points $\y_1, \dots, \y_n \in \RR^m$. }
  \vspace{6pt}

  \textbf{(1)}
  Compute the weights matrix $W$ via Eq.~\eqref{eq: weight matrix} and its corresponding degree matrix $D$. \;
  
  \BlankLine
  \textbf{(2)}
  Calculate the random-walk normalized graph Laplacian $L_{RW} := I - D^{-1}W$. \;

  \BlankLine
  \textbf{(3)}
  Compute the eigenvectors $\varphi_1, \dots, \varphi_m$ of $L_{RW}$ that correspond to the first $m$ nonzero eigenvalues $\lambda_1 \leq \cdots \leq \lambda_m$.\;

  \BlankLine
  \textbf{(4)}
  Map $\x_i$ into $\RR^m$ via $\x_i \mapsto \y_i = \left(\varphi_1(\x_i), \dots, \varphi_m(\x_i)\right)$.

  \BlankLine
  \textbf{(5)}
  {\Return $\y_1, \dots, \y_n$}.
\end{algorithm}    

It is known that if one uses the Gaussian kernel to construct the weight matrix $W$ then the random-walk Laplacian converges in a pointwise sense to the Laplace--Beltrami operator on the manifold $\M$.  A precise statement is as follows.

\begin{theorem}[\cite{Singer2006}]\label{thm: convergence of standard graph Laplacian to Laplace-Beltrami operator}
Let $\M$ be a compact $d$-dimensional Riemannian submanifold of $\RR^D$ with Laplace--Beltrami operator  $\Delta_\M f := -\mathrm{div}(\nabla_\M f)$. For any smooth function $f:\M \to \RR$ and given points $\x_1, \dots, \x_n$ sampled independently and uniformly from $\M$, we have that for all sufficiently small $\varepsilon$ and as $n \to \infty$
\begin{align} \label{eq:LRW_convergence}
    \frac{4}{\varepsilon}(L_{RW} f)(\x_i)
    =
    \Delta_\M f(\x_i)
    +
    O\left( \varepsilon \right)
    +
    O_P
    \left(
        \frac{1}{n^{1/2}\varepsilon^{1/2 + d/4}}
    \right),
\end{align}
\noindent where $L_{RW}$ is the random-walk graph Laplacian based on the Gaussian kernel $K(x,y)=\exp(-\|x - y\|^2/\varepsilon)$. 
\end{theorem}
Here $O(\varepsilon)$ denotes a quantity less than $C \varepsilon$ 
where $C$ is a constant that depends on $\M$, $f$ and $x_i$ but not on  $\varepsilon$ or $n$, while $O_P(\cdot)$ denotes the order in probability.  The theorem can be extended to non-uniform sampling.
In that case, the random-walk graph Laplacian converges to a weighted Laplacian (or Fokker--Planck) operator on $\mathcal{M}$, which has an additional drift term. 
\begin{theorem}[\cite{COIFMAN20065}]\label{thm: convergence of standard graph Laplacian to Fokker-Planck operator}
Assume the setting of Theorem~\ref{thm: convergence of standard graph Laplacian to Laplace-Beltrami operator}, except that the points $\x_1, \dots, \x_n \in \M$ are now sampled from a probability density $\rho(x)$ with respect to the uniform measure on $\M$. Then for all sufficiently small $\varepsilon$ and as $n \to \infty$ we have that
\begin{align}
    \frac{4}{\varepsilon}
    (L_{RW} f)(\x_i)
    =
    \Delta_\M f(\x_i)
    -
    2\left\langle \nabla_\M \log \rho(\x_i),\, \nabla_\M f(\x_i) \right\rangle
    + O(\varepsilon)
+     O_P
    \left(
        \frac{1}{n^{1/2}\varepsilon^{1/2 + d/4}}
    \right).
\end{align}
\end{theorem}  

These theorems have been generalized to prove spectral consistency, where the eigenvalues and eigenvectors of the graph Laplacian converge to the eigenvalues and eigenfunctions of the Laplace--Beltrami operator \citep{vonLuxburgBelkinBousquet2008, BELKIN20081289, GarciaTrillosSlepcevVariational, GarciaTrillos2019, calder2022improved, 
Cheng2021EigenconvergenceOG}. In this work, we focus on the pointwise consistency of our method. 

\section{Method and Main Results}\label{sec:method}
Here we present our approach to incorporating symmetry invariance into spectral embedding.  
We begin by defining $G$-invariant kernels and presenting three broadly applicable constructions (Section~\ref{sec:G-invariant kernels}). We then describe the G-invariant spectral embedding algorithm (Section~\ref{sec: G-invariant spectral embedding}), followed by our main theoretical result.  It is a pointwise convergence theorem showing that graph Laplacians built from $G$-invariant kernels converge to explicit second-order differential operators on the quotient manifold $\N = \M/G$ (Section~\ref{sec: convergence of invariant-kernel graph Laplacian}). We conclude by analyzing the special case of constant orbit-volume density (Section~\ref{sec: special case constant density}) and illustrate the improved rate with a numerical example (Section~\ref{sec: numerical example}).

\subsection{Group Invariant Kernels}\label{sec:G-invariant kernels}
The following definition formalizes the notion of a group invariant kernel.

\begin{definition}
Let $G$ be a compact Lie group acting on $\RR^D$, and let $K_G \colon \RR^D \times \RR^D \to \mathbb{R}$ be a symmetric  kernel.
We say that $K_G$ is a \textbf{$G$-invariant kernel} if
\[
K_G(\alpha\cdot x,\beta\cdot y)=K_G(x,y)
\]
for all $x,y\in\RR^D$ and $\alpha,\beta\in G$. 
\end{definition}
In particular,  we focus on the following three constructions of $G$-invariant kernels.
\begin{proposition}
Let $G$ be a compact Lie group acting by isometries on $\RR^D$. The following three kernels are $G$-invariant:
\begin{enumerate}
    \item The \textbf{minimum kernel}:
    \begin{equation}\label{eq: min Gaussian kernel}
        K_{\mathrm{min}}(\x,\y) 
        :=
        \max_{\alpha \in G}
        K(\x, \alpha \cdot \y) 
        =
        \exp
        \left(
            -\min_{\alpha \in G} \| \x - \alpha\cdot \y \|^2 /\varepsilon
        \right). 
    \end{equation}
    \item   The \textbf{integral kernel}:
    \begin{equation}\label{eq: integral Gaussian kernel}
        K_{\mathrm{int}}(\x,\y) 
        :=
        \int_G K(\x, \alpha \cdot \y) d\eta(\alpha)
        =
        \int_G
        \exp
        \left(
            -\| \x - \alpha \cdot \y \|^2 / \varepsilon
        \right)
        d\eta(\alpha),
\end{equation}
where $\eta$ denotes the Haar measure on $G$ normalized so that $\vol(G) = \int_G 1 d\eta(\alpha) = 1$.

    \item  The \textbf{invariant features kernel}: 
\begin{equation}\label{eq: invariant feature kernel}
    K_{\mathrm{IF}}(\x,\y)
    :=
    K(\phi(\x), \phi(\y))
    =
    \exp
    \left(
        -\| \phi(\x) - \phi(\y)\|^2 / \varepsilon
    \right),
\end{equation}
    where $\phi: \RR^D \to \mathbb{R}^E$ is a $G$-invariant map, i.e. $\phi(\alpha \cdot \x) = \phi(\x)$ for all $\alpha \in G, \x \in \RR^D$.
\end{enumerate}
\end{proposition}
\begin{proof}
    Follows directly from the definitions.
\end{proof}
The minimum and integral kernels are completely determined by the action of $G$ on $\RR^D$. The invariant features kernel, on the other hand, further depends on the choice of a $G$-invariant map $\phi$. 
The following examples illustrate each of the three $G$-invariant kernels in two different settings.
\begin{example}[Point lists in $\RR^d$, $G=\mathrm{O}(d)$]\label{ex: Gram matrix}
Let the points $\x_1, \ldots, \x_n \in \RR^d$ be given as rows of a matrix $X \in \RR^{n \times d}$. Any rotation (and rotoreflection) of the points can be written as $XR^\top$ where $R \in \mathrm{O}(d)$ is an orthogonal matrix. The minimum kernel then corresponds to the orthogonal Procrustes problem for $X$ and another point list $Y$, which has a closed-form solution \citep{Schonemann1966Procrustes}. 
Meanwhile for the integral kernel, one needs to approximate
\begin{align}
    \int_{\mathrm{O}(d)}
    \exp
    \left(
        -\|X - YR^T\|^2 / \varepsilon
    \right)
    d\eta(R).
\end{align}
Similar integrals without the exponential appear in steerable PCA works \citep{zhao2016fast, fraiman2026so}.

For an instance of the invariant features kernel, the Gram matrix of $X$ defines a map $\phi \colon \RR^{n \times d} \to \RR^{n \times n}$ via $\phi(X) = X X^\top$; this is $\mathrm{O}(d)$-invariant since for all $R \in \mathrm{O}(d), X \in \RR^{n \times d}$ we have $\phi(XR^\top) = (XR^\top)(XR^\top)^\top = X X ^\top= \phi(X)$. 
In fact, every $\mathrm{O}(d)$-invariant function factors through the Gram matrix, as the first fundamental theorem of invariant functions for the orthogonal group~{\cite[Theorem~2.9A]{Wey46}} shows: 
A function $f \colon \RR^{n \times d} \to \RR$ is invariant under the action of the orthogonal group $\mathrm{O}(d)$ (i.e., $f(X) = f(XR^\top)$ for all $R \in \mathrm{O}(d)$, $X \in \RR^{n\times d}$) if and only if there exists a function $g \colon \RR^{n \times n} \to \RR$ such that $f(X) = g(X X^\top).$
\end{example}

\begin{example}[2D images, $G = \mathrm{SO}(2)$]\label{ex:bispectrum}

Consider 2D images as square integrable functions $X \colon \RR^2 \to \RR$ and let $R \in \mathrm{SO}(2)$ act on  $X$ via $(R \cdot X)(x) = X(R^\top x)$. The minimum kernel corresponds to 2D alignment of two functions, while the integral kernel is given by 
\begin{align}
    \int_{\mathrm{SO}(2)}
    \exp
    \left(
        -\tfrac{\| X - R\cdot Y\|_{L^2}^2}{\varepsilon}
    \right)
    d\eta(R) =
    \frac{1}{2\pi}
    \exp
    \left(
        -
        \tfrac{\|X\|_{L^2}^2 + \|Y\|_{L^2}^2}{\varepsilon}
    \right)
\int_0^{2\pi}
\exp\!\left(
  \frac{2}{\varepsilon}
  \int_{\mathbb{R}^2} X(x)\, Y(R_\theta^\top x)\, dx
\right)
d\theta.
\end{align}

For the invariant features kernel, one can define a set of rotational invariants based on a Fourier--Bessel decomposition:
\begin{align}
    X(r, \theta) = \sum_{m=0}^{+\infty} \sum_{k=1}^{+\infty} a_{k,m} \psi_{k, m}(r, \theta),
\end{align}
where $\psi_{k, m}(r, \theta)$ are the basis functions~\citep{zhao2013}. The \emph{(rotational) bispectrum} of $X$ is given by $\phi(X)$ with elements
\(
   a_{k_1,m_1} a_{k_2,m_2} a_{k_3, m_1+m_2}^*
\)
where $k_1 = 1, \ldots, k_{\max}(m_1)$, $k_2 = 1, \ldots, k_{\max}(m_2)$, $k_3 = 1, \ldots, k_{\max}(m_1+m_2)$, $m_1 = 0, \ldots, \lfloor m_{\max} / 2 \rfloor$ and $m_2 = 0, \ldots, m_{\max} - m_1$.
Here $m_{\max}$ is an angular-frequency cutoff, and $k_{\max}(m)$ is the
number of radial Fourier--Bessel coefficients retained for angular frequency
$m$.
The bispectrum is an $\mathrm{SO}(2)$-invariant map \citep{zhao2014}. 
\end{example}

\subsubsection*{Computation of Invariant Kernels}

Computation of the minimum kernel requires evaluating $\min_{\alpha \in G} \|x - \alpha \cdot y\|^2$, which is a non-convex global optimization problem. 
For $G = \mathrm{SO}(n)$ and $G = \mathrm{O}(n)$ acting on $\RR^{n \times m}$ by left multiplication, there are closed-form solutions based on the SVD \citep{Kabsch1976Rotation, umeyama1991least, Schonemann1966Procrustes}.
When different actions of $G = \mathrm{SO}(n)$ or $G = \mathrm{O}(n)$ are considered, one obtains generally non-trivial alignment problems. For $G = \mathrm{SO}(2)$, one may compute the distance as a trigonometric polynomial and solve the minimization using univariate search or Fourier methods. For $G = \mathrm{SO}(3)$, deterministic low-discrepancy designs, such as Fibonacci and super-Fibonacci samplings \citep{Alexa2022SuperFibonacciSF}, provide efficient alternatives to random sampling. 
For general compact Lie groups, there exist branch-and-bound algorithms 
\citep{campbell2016gogma, Hartley2013}. Manifold optimization methods \citep{boumal2023intromanifolds}  include gradient descent and Newton's method over general Lie groups; however, these methods are not guaranteed to find a global minimum in general.

For the integral kernel, a naive quadrature rule requires $O(L^{p})$ operations per pair of data points, where $L$ represents a resolution parameter and $p$ is the dimension of $G$. For $G = \mathrm{SO}(2)$, uniform trapezoidal rules on the circle yield exponentially fast convergence for smooth integrands \citep{trefethen2014exponentially}.
For $G = \mathrm{SO}(3)$, one can apply Monte Carlo sampling with exact Haar draws from quaternion parameterizations \citep{Shoemake1992} or quasi–Monte Carlo rules using low-discrepancy point sets \citep{Alexa2022SuperFibonacciSF}. 
For a general compact Lie group, there exist quadrature schemes based on FFT-type algorithms \citep{maslen2004sampling,kostelec2008ffts, potts2009fast}.

 Finally, computation of the invariant features kernel depends on the choice of $G$-invariant map $\phi$. A generic cost includes two evaluations of $\phi$ plus $O(E)$ operations to evaluate $\| \phi(\x) - \phi(\y)\|^2$, where $E$ is the output dimension of $\phi$.
 Returning to the examples above, the computation of the Gram matrix in Example~\ref{ex: Gram matrix} requires $O(n^2d)$ operations. Moreover, the bispectrum of a 2D image in Example~\ref{ex:bispectrum} can be calculated in $O(L^3)$ time~\citep{zhao2013}, where each image is represented on a discrete $L \times L$ grid. 
In general, we also note it may sometimes be possible to evaluate $\| \phi(x) - \phi(y) \|^2$ without forming $\phi(x)$ and $\phi(y)$ explicitly, i.e. for some $G$-invariant maps there may be a kernel trick available.
 \subsection{Spectral Embedding with Group Invariant Kernels}\label{sec: G-invariant spectral embedding}
 Given a dataset with a known symmetry group $G$, our proposal is to replace the generic affinity kernel $K$ with a $G$-invariant kernel $K_G$ in the construction of the graph Laplacian $L_{RW}$ (Eqs.~\eqref{eq: weight matrix}-\eqref{eq: RW graph Laplacian acting on functions}). In principle, any $G$-invariant kernel may be used in this construction. The resulting graph Laplacian $L_{RW}$ is then used directly in Algorithm~\ref{alg:spectral embedding}, with no other modification to the spectral embedding procedure. Theoretical guarantees for the case of the minimum, integral, and invariant features kernels introduced in Section~\ref{sec:G-invariant kernels} are established in Theorem~\ref{thm: convergence of G-invariant graph Laplacian} below.

\subsection{Convergence of the Graph Laplacian with Group Invariant Kernels}\label{sec: convergence of invariant-kernel graph Laplacian}

Let $G$ be a compact Lie group acting by isometries on $\RR^D$, and $\M$  a compact $d$-dimensional Riemannian submanifold of $\RR^D$ that is $G$-invariant, i.e., $\alpha \cdot \M = \M$ for all $\alpha \in G$. It follows that the restricted action of $G$ on $\M$ is by isometries.  We further assume that this action is smooth and free.
Denote the corresponding quotient manifold by $\N = \M/G$. 
It inherits a Riemannian metric naturally.
See Appendix~\ref{app:quotient_manifolds} for details.
Our analysis focuses on group invariant functions, defined as follows.
\begin{definition}
We say that $f \colon \M \to \RR$ is a \textbf{$\mathbf{G}$-invariant function} if $f(\alpha \cdot x) = f(x)$ for all $\alpha \in G$, $x \in \M$. In this case, the unique function $\overline{f} \colon \N \to \RR$ such that $\overline{f}([x]) = f(x)$ for all $x \in \M$ is called the \textbf{induced function} on
$\N$.
\end{definition}
Restricting to $G$-invariant functions is a natural choice for datasets with known symmetries, since we aim to treat symmetry-related observations as being the same. For example, the underlying molecular structure of projection images in cryo-EM is unaffected by in-plane rotations, so meaningful quantities should be rotation invariant. 
Next, we follow \cite[Section II.3.3]{helgason2000groups} and introduce a function $\delta$ on $\M$ which measures the relative volume of the orbits of $G$. 

\begin{definition}\label{def: density function} 
Given $x \in \M$, let $dV_{G \cdot x}$ denote the Riemannian measure on the orbit $G \cdot x$ induced by the ambient metric on $\M$. Fix a left-invariant Haar measure $d\eta$ on $G$ such that $\vol(G) = \int_G 1 d\eta(\alpha) = 1$. 
The \textbf{density function} $\delta \colon \M \to \RR$ is defined by the formula
\begin{align}\label{eq: density function}
    dV_{G\cdot x}= \delta(x) d\eta, \quad x \in \M,
\end{align} 
\end{definition} 
\noindent where we identify $G\cdot x$ with $G$ via the orbit map $\theta^{(x)} \colon G \to G\cdot x$ given by $\theta^{(x)}(\alpha) = \alpha \cdot x$.
To see that $\delta$ is well-defined, note that the orbit map is a diffeomorphism since $G$ acts freely on $\M$, which provides an identification of $G$ with $G\cdot x$ for each $x \in \M$, and it follows that $dV_{G\cdot x}$ and $d\eta$ must be proportional by uniqueness of the Haar measure on $G$.
It is clear from the definition that $\delta$ is a $G$-invariant function, so it descends to a function $\overline{\delta}$ on the quotient manifold $\N$.

We are now ready to state the main theoretical result regarding the convergence of the graph Laplacian based on a group invariant kernel.
\begin{theorem}[Main theoretical result]\label{thm: convergence of G-invariant graph Laplacian}
Let $G$ be a compact Lie group of dimension $p$ acting isometrically on $\RR^D$. Let $\M$ be a connected compact $d$-dimensional Riemannian submanifold of $\RR^D$ without boundary. Assume that $\M$ is $G$-invariant and that $G$ acts smoothly and freely on $\M$. Denote the Riemannian quotient manifold $\M / G$ by $\N$. Let $\overline{\delta}$ denote the  function on $\N$ induced by the density function $\delta$ defined in Eq. \eqref{eq: density function}, which we assume to be smooth. Let data points $\x_1, \dots, \x_n$ be i.i.d. samples from  the uniform measure $dV_\M$ on $\M$. Let $f\colon \M \to \RR$ be a smooth $G$-invariant function and  $\overline{f} \colon \N \to \RR$ the corresponding induced function.   Let $L_{RW}$ be the normalized graph Laplacian defined in Eq.~\eqref{eq: RW graph Laplacian} using a $G$-invariant kernel. Then for all sufficiently small $\varepsilon$ and as $n \to \infty$  we  have
\begin{equation}\label{eq: convergence with improved rate}
    \frac{4}{\varepsilon} (L_{RW} f)(\x_i) = \D \overline{f}([\x_i]) + O(\varepsilon) + O_P\left(\frac{1}{n^{1/2}\varepsilon^{1/2 + (d-p)/4}}\right),
\end{equation}
where $\D$ is a second-order differential operator on $\N$ which depends on the choice of $G$-invariant kernel as follows:
    \begin{enumerate}
    \item For the minimum kernel in Eq. \eqref{eq: min Gaussian kernel},
    \begin{align}\label{eq: D min kernel}
    \D = \Delta_\N- 2\left\langle \nabla_\N \log\overline{\delta}, \nabla_\N(\cdot) \right\rangle.
    \end{align} 

    \item For the integral kernel in Eq. \eqref{eq: integral Gaussian kernel},
    \begin{align} \label{eq: D integral kernel}
    \D = \Delta_\N  - \left\langle \nabla_\N \log\overline{\delta}, \nabla_\N(\cdot) \right\rangle. \end{align} 

    \item For the invariant features kernel in Eq.~\eqref{eq: invariant feature kernel},
    \begin{align}\label{eq: D invariant features kernel}
        \D = \bar{\phi}^{\ast} \D_{\mathrm{im}\phi}.
    \end{align}
    Here $\phi \colon \M \to \mathbb{R}^E$ is a smooth $G$-invariant map such that
    the induced map $\bar{\phi} \colon \N \to \mathrm{im}\phi$ is a diffeomorphism.
    The operator $D_{\mathrm{im}\phi}$ is defined on $\mathrm{im}\phi$ by
    \begin{align}
        \D_{\mathrm{im}\phi}
        =
        \Delta_{\mathrm{im}\phi}
        -
        2\left\langle
        \nabla_{\mathrm{im}\phi}\log p_\phi,
        \nabla_{\mathrm{im}\phi}(\cdot)
        \right\rangle,
    \end{align}
    where $p_\phi$ is the density of the pushforward measure $\phi_\ast(dV_\M)$ with
    respect to the Riemannian volume measure $dV_{\mathrm{im}\phi}$ induced on
    $\mathrm{im}\phi$ by the ambient Euclidean metric, that is, $ \phi_\ast(dV_\M) = p_\phi\, dV_{\mathrm{im}\phi}.$
    
    \end{enumerate}
\end{theorem}
Theorem~\ref{thm: convergence of G-invariant graph Laplacian} is proved in Section~\ref{sec: Proof of main thm}.

\begin{remark}
\label{rem:metric_distortion}
The invariant features map $\phi$ need not be an isometry. In that case, spectral embedding based on the invariant features kernel may distort distances between data points.  Our analysis takes this into account.
\end{remark} 

The theorem can be extended to non-uniform sampling of data points, where the operator $D$ acquires an additional first-order term depending on the average over $G$ of the sampling density.
\begin{corollary} \label{cor:improved_convergence_rate}
     Assume the conditions of Theorem~\ref{thm: convergence of G-invariant graph Laplacian}, but with data points $x_1, \dots, x_n$ sampled from a smooth probability density $\rho$ with respect to the uniform measure on $\M$. Let $\tilde{\rho}$ be the function on $\N$ defined by
     \begin{equation}
         \tilde{\rho}([x]) = \int_G \rho(\alpha \cdot x) d\eta(\alpha).
     \end{equation}
     Then for all sufficiently small $\varepsilon$ and as $n \to \infty$ we  have
     \begin{equation}\label{eq: convergence with improved rate and non-uniform density}
    \frac{4}{\varepsilon} (L_{RW} f)(\x_i) = \D_\rho \overline{f}([\x_i]) 
    + O(\varepsilon) + O_P\left(\frac{1}{n^{1/2}\varepsilon^{1/2 + (d-p)/4}}\right),
\end{equation}
    where $\D_\rho$ adds a first-order term to the operator $D$ in Theorem~\ref{thm: convergence of G-invariant graph Laplacian} as follows:
    \begin{enumerate}
    \item For the minimum kernel in Eq.~\eqref{eq: min Gaussian kernel} and $D$ in Eq.~\eqref{eq: D min kernel}, 
    \begin{equation}
        \D_\rho = \D -2\left\langle \nabla_\N \log \tilde{\rho}, \nabla_\N (\cdot) \right\rangle.
    \end{equation}
    
    \item For the integral kernel in Eq.~\eqref{eq: integral Gaussian kernel} and $D$ in Eq.~\eqref{eq: D integral kernel}, 
    \begin{equation}
        \D_\rho = \D  -2\left\langle \nabla_\N \log \tilde{\rho}, \nabla_\N (\cdot) \right\rangle.
    \end{equation}

    \item For the invariant features kernel in Eq.~\eqref{eq: invariant feature kernel} and $D$ in Eq.~\eqref{eq: D invariant features kernel}, 
    \begin{equation}
        \D_\rho = \D -2 \bar{\phi}^{\ast} \left\langle
        \nabla_{\mathrm{im}\phi}\log(\tilde{\rho}\circ \bar{\phi}^{-1}),
        \nabla_{\mathrm{im}\phi}(\cdot)
        \right\rangle.
    \end{equation}
    \end{enumerate}
\end{corollary}
\begin{proof}
For the minimum kernel,  Eq.~\eqref{eq: Fubini integral on M} below shows that $\bar{\delta} \mapsto \bar{\delta}\tilde{\rho}$. 
For the invariant features kernel,
given any measurable set $A \subseteq \mathrm{im}\phi$ we have
\begin{align}\label{eq: pushforward measure if kernel}
   \phi_\ast(\rho\, dV_\M)(A) = \int_{\pi^{-1}(\bar\phi^{-1}(A))} \rho\, dV_\M  = \int_{\bar\phi^{-1}(A)} \tilde\rho([x])\, \bar{\delta}([x])\, dV_\N([x]) = \int_A (\tilde\rho \circ \bar{\phi}^{-1})(y)d\mu(y),
\end{align}
where $\mu = \bar\phi_\ast(\bar \delta dV_\N)$ is the pushforward under $\bar\phi$ of the measure $\bar \delta dV_\N$.  Taking $\rho = 1$  gives $\phi_\ast(\rho\, dV_\M)(A) = \int_A d\mu$, which implies that $\mu = \phi_\ast(\rho\, dV_\M) = p_\phi dV_{\mathrm{im}\phi}$, so Eq.~\eqref{eq: pushforward measure if kernel} becomes
\begin{align}
   \phi_\ast(\rho\, dV_\M)(A)
   = \int_A  (\tilde\rho \circ \bar\phi^{-1})(y)\, p_\phi(y)\, dV_{\mathrm{im}\,\phi}(y).
\end{align}
Hence, the pushforward measure satisfies $\phi_\ast(\rho dV_\M) = p_\phi \cdot (\tilde{\rho} \circ \bar{\phi}^{-1}) dV_{\mathrm{im}\phi}$,
meaning that $p_\phi \mapsto p_\phi \cdot(\tilde{\rho} \circ \bar{\phi}^{-1})$. The result then follows for these two kernels by the additivity of $\log$ and the linearity of the inner product. 
For the integral kernel, \citet{ROSEN2024} show that $\Delta_\M$ picks up the extra term $- 2\langle\nabla_\M \log \pi^\ast \tilde{\rho}, \nabla_\M (\cdot)\rangle$, which equals $-2\left\langle \nabla_\N \log \tilde{\rho}, \nabla_\N (\cdot) \right\rangle$ since the quotient map $\pi \colon \M \to \N$ is a Riemannian submersion. 
\end{proof}

\begin{remark}
    If $\rho$ is a $G$-invariant function, then $\tilde{\rho}$ is precisely the induced function $\bar \rho$ on $\N$.
\end{remark}
   
\begin{remark}
   The error bound in Eqs.~\eqref{eq: convergence with improved rate} and~\eqref{eq: convergence with improved rate and non-uniform density} is equivalent to a sample complexity bound for spectral methods. Balancing the two error terms shows that to achieve a total error of order $\eta$
   it suffices to take $\varepsilon < \eta$ and 
   \begin{equation}
       n =  \Omega \left(\varepsilon^{-3 - (d-p)/2}\right).
   \end{equation}
    Hence the use of group invariant kernels improves the sample complexity for spectral embedding methods by the dimension of the Lie group $G$.
\end{remark}

\begin{remark}\label{rmk: improved convergence}
    If $\nabla_\N \bar{f} ([\x_i]) =0$ (equivalently, if  $\nabla_\M f (\x_i) =0$) one can show that the convergence rate in Eqs.~\eqref{eq: convergence with improved rate} and~\eqref{eq: convergence with improved rate and non-uniform density} improves by an extra factor of $\varepsilon^{1/2}$, see \citep[Eq. (3.17)]{Singer2006}. Indeed, the variance error term becomes
    \begin{equation}
        O_P\left(\frac{1}{n^{1/2}\varepsilon^{(d-p)/4}}\right).
    \end{equation}
    In this case, the sample complexity is $n=  \Omega \left(\varepsilon^{-2 - (d-p)/2}\right)$.
\end{remark}
\subsection{Special Case of Constant Density $\delta$}\label{sec: special case constant density}

In the special case that the density function $\delta$ defined by Eq.~\eqref{eq: density function} is constant, our main result simplifies considerably.  Eqs.~\eqref{eq: convergence with improved rate},~\eqref{eq: D min kernel} and~\eqref{eq: D integral kernel} show that in this case the graph Laplacian $L_{RW}$ based on the minimum kernel or the integral kernel converges to the Laplace--Beltrami operator $\Delta_\N$ on the quotient. 
Moreover, in this setting we obtain a further relation between the Laplace--Beltrami operators $\Delta_\M$ and $\Delta_\N$, and 
the $G$-invariant eigenfunctions of $\Delta_\M$ are in bijection with the eigenfunctions of $\Delta_\N$.
To make the discussion precise, define the \emph{projection} of $\Delta_\M$ as the differential operator $P(\Delta_\M)$ on $\N$ such that
    \begin{equation}\label{eq: def of projection of Laplace-Beltrami operator}
        \pi^\ast\left(P(\Delta_\M) h\right) = \Delta_\M ( \pi^\ast h )
    \end{equation}
 for any smooth function $h \colon \N \to \RR$, where $\pi \colon \M \to \N$ is the quotient map. 
 The following lemma gives an explicit formula for $P(\Delta_\M)$. See Proposition 3.1 and Remark (c)(ii) in \citep{Le2001} for a proof.

\begin{lemma}\label{lem: radial part of the Laplace-Beltrami operator}

The projection of $\Delta_\M$ onto $\N$ is given by the formula
\begin{equation}\label{eq:projection of Laplace-Beltrami operation}
    P(\Delta_\M) = \Delta_\N - \left\langle \nabla_\N \log\overline{\delta}, \nabla_\N(\cdot) \right\rangle.
\end{equation}
\end{lemma}
In the special case that $\delta$ is constant, the first-order term in Eq.~\eqref{eq:projection of Laplace-Beltrami operation} vanishes and the projection of $\Delta_\M$ is precisely the Laplace--Beltrami operator on the quotient.
\begin{corollary}\label{cor:clean_descent}
Assume the density function $\delta$ 
is constant. Then $P(\Delta_\M) = \Delta_\N.$
\end{corollary}

In this regime the spectral analysis of $G$-invariant functions on~$\M$ agrees \emph{exactly} with the spectral analysis of functions on the quotient~$\N$,
simplifying both theory and computation. Concretely, eigenpairs of $\Delta_\M$ restricted to $G$-invariant functions are in bijection with eigenpairs of $\Delta_\N$.

\begin{corollary}[Eigenfunction correspondence]
\label{cor:eigen_descend}
Assume that the density function $\delta$ is constant,
let $f\in C^\infty(\M)$ be $G$–invariant and write
$f=\pi^\ast\bar f$ for $\bar f\in C^\infty(\N)$. Then $\Delta_\M f= \lambda f$ if and only if
$\Delta_\N\bar f =\lambda\bar f$.
\end{corollary}
\begin{proof}
By Eq.~\eqref{eq: def of projection of Laplace-Beltrami operator} and Corollary~\ref{cor:clean_descent}, $\Delta_\M f= \lambda f$ if and only if $\pi^\ast(\Delta_\N\bar f) =\lambda~\pi^\ast \bar f $. The result follows from the surjectivity of $\pi$.
\end{proof}

We illustrate these results with the following example.

\begin{example}\label{ex:so2_on_so3}
Consider $\M = \mathrm{SO}(3) \subset \RR^{3\times 3}$ with the induced Frobenius metric\footnote{This metric differs from the more standard bi-invariant metric in \cite{chirikjian2016harmonic} by $ g_\mathrm{Frob} = 2 g_\mathrm{bi-inv}$, which implies $\Delta_{\mathrm{SO}(3)}^\mathrm{Frob} =  \Delta_{\mathrm{SO}(3)}^\mathrm{bi-inv}$/2.} from $\RR^{3\times 3}$. Let
$G = \mathrm{SO}(2)$ act on $SO(3)$  by $R \cdot A = AR^T$ for $R \in \mathrm{SO}(2)$ and $A \in \mathrm{SO}(3)$, where $\mathrm{SO}(2)$ is embedded in $\mathrm{SO}(3)$ as rotations about the $z$-axis. 
The quotient manifold $\N=\mathrm{SO}(3)/\mathrm{SO}(2)$ is diffeomorphic to the 2-sphere $\mathbb{S}^2$ \citep[Example 21.19(a)]{lee2003introduction}.
Moreover, all orbits of this action have volume $2\pi\sqrt{2}$ since they are all isometric to $\mathrm{SO}(2)$, so the density function $\delta$ is constant.
If we parametrize $R \in \mathrm{SO}(2)$ by the rotation angle $\theta \in [0,2\pi)$ and $A \in \mathrm{SO}(3)$ with $ZYZ$ Euler angles
$\alpha \in [0, 2\pi), \beta \in [0, \pi],  \gamma \in [0, 2\pi)$,
then $R$ acts on $A$ by $
    R(\theta)\cdot A(\alpha,\beta,\gamma)
    =
    A(\alpha,\beta,\gamma - \theta).$
It follows that the $\mathrm{SO}(2)$-invariant functions
$f:\mathrm{SO}(3)\to\RR$ are precisely those which do not depend on the Euler
angle $\gamma$.
The Laplace--Beltrami operator on $\mathrm{SO}(3)$ is \citep[Eq. (9.22)]{chirikjian2016harmonic}
\begin{align}
\Delta_{\mathrm{SO}(3)}
= \frac{1}{2}\left(
-\frac{\partial^2}{\partial\beta^2}
-\cot\beta \frac{\partial}{\partial \beta}
-\frac{1}{\sin^2\beta}
\left(
\frac{\partial^2}{\partial\alpha^2}
-2\cos\beta \frac{\partial^2}{\partial\alpha \partial\gamma}
+\frac{\partial^2}{\partial\gamma^2}
\right) \right).
\end{align}
Since a $\mathrm{SO}(2)$-invariant function $f$ does not depend on $\gamma$, the $\gamma$-derivatives vanish, and Corollary~\ref{cor:clean_descent} gives
\begin{align}
    \Delta_{\mathrm{SO}(3)}f
    =
    \frac{1}{2}
    \left(
        -\frac{\partial^2 f}{\partial\beta^2}
        -\cot\beta \frac{\partial f}{\partial \beta}
        -\frac{1}{\sin^2\beta} \frac{\partial^2 f}{\partial\alpha^2}
    \right)
    =
    \pi^\ast(\Delta_{\mathbb{S}^2}\overline{f})
\end{align}
where $\mathbb{S}^2$ is parameterized by the polar angle $\beta$ and the azimuthal angle $\alpha$. 
The eigenfunctions of $\Delta_{\mathrm{SO}(3)}$ are linear combinations of the
unitary irreducible representations
\begin{align}
U^\ell_{mn}(A(\alpha,\beta,\gamma))
=
 e^{-im\alpha}P^\ell_{mn}(\cos\beta)e^{-in\gamma},
\end{align}
with eigenvalue $\ell(\ell+1)/2$,
where $P^\ell_{mn}$ denotes the generalized associated Legendre polynomials \citep[Eq. (9.41)]{chirikjian2016harmonic}. The $\mathrm{SO}(2)$-invariant
eigenfunctions, those with $n=0$, are therefore linear combinations of
\begin{align}
U^\ell_{m0}\left(A(\alpha,\beta,\gamma)\right) 
= e^{-im\alpha}P^\ell_{m0}(\cos\beta)
= (-1)^m
\sqrt{\frac{4\pi}{2\ell+1}}\overline{Y^m_\ell}(\beta,\alpha),
\end{align}
where $Y^m_\ell$ denotes the spherical harmonic on the unit sphere \citep[Eq. (9.42)]{chirikjian2016harmonic}, in agreement with Corollary~\ref{cor:eigen_descend}.
In particular, for $\ell =1$ and $m=0$ we have $U^1_{00}\left(A(\alpha,\beta,\gamma)\right) = P^1_{00}(\cos\beta) = \cos\beta$, an eigenfunction of $\Delta_{\mathrm{SO}(3)}$ with eigenvalue $1(1 +1)/2 = 1$.
\end{example}

\subsection{A Numerical Example}\label{sec: numerical example}
We illustrate the improved convergence rate in Eq.~\eqref{eq: convergence with improved rate} compared to spectral embeddings that are not symmetry-aware with a numerical example in the special case of a constant density  $\delta$. Examples with non-constant $\delta$ that illustrate Theorem~\ref{thm: convergence of G-invariant graph Laplacian} are in Section~\ref{sec:experiments}.

Consider the action of $G = \mathrm{SO}(2)$ on $\M = \mathrm{SO}(3) \subseteq \RR^{3\times 3}$ from 
Example~\ref{ex:so2_on_so3} and the
$G$-invariant function $f(R) =  \cos\beta(R),$  where $0 \leq \beta(R) \leq \pi$ is the second Euler angle in the $ZYZ$ parametrization of $R \in \mathrm{SO}(3)$. The function $f$ satisfies $\Delta_\M f = f$ (Example~\ref{ex:so2_on_so3}), and since $\delta$ is constant, Corollary~\ref{cor:clean_descent} implies that $\Delta_\N \overline{f} =  \overline{f}$. Therefore, 
at $R_0 = I_3$, we have $\Delta_\M f( {I_3}) = 1$ and $\Delta_\N \overline{f}({[I_3]}) = 1$. By writing the Riemannian metric elements in terms of Euler angles and using the formula for the Riemannian gradient, a straightforward calculation shows that $\nabla_\M f({I_3}) = 0$.

We uniformly sample $n = 10000$ data points from the Haar distribution on $\mathrm{SO}(3)$ and $m = 200$ group elements from the Haar distribution on $\mathrm{SO}(2)$. We then use the Euclidean, minimum, and integral kernels to compute $(L_{RW} f)(I_3)$ for each kernel. By Theorem~\ref{thm: convergence of standard graph Laplacian to Laplace-Beltrami operator}, for the classical Euclidean kernel, $\frac{4}{\varepsilon} (L_{RW} f)(I_3)$ approximates $\Delta_\M f({I_3}) = 1$. Meanwhile, by Theorem~\ref{thm: convergence of G-invariant graph Laplacian}, for the minimum and integral kernels, $\frac{4}{\varepsilon}(L_{RW} f)(I_3)$ approximates $D \overline{f}([I_3]) = \Delta_\N \overline{f}({[I_3]}) = 1$ since $\delta$ is constant. 
Figure~\ref{fig:so3_singlepoint_results} shows a logarithmic plot of the mean absolute error over $T = 1000$ trials between $\frac{4}{\varepsilon} (L_{RW} f)(I_3)$ and $1$ versus $\varepsilon$ for each kernel.
As explained in Remark~\ref{rmk: improved convergence}, the condition $\nabla_\M f(I_3) = 0$ implies that the predicted slopes of the error in the variance-dominated region (small values of $\varepsilon$) are $-0.75$ for the Euclidean kernel and  $-0.5$ for the minimum and integral kernels. Linear fits give slopes of $-0.853$, $-0.503$ and $-0.505$ for the Euclidean, minimum and integral kernels, respectively. For the minimum and integral kernels, the fitted slopes are consistent with the corresponding predictions. By contrast, for the Euclidean kernel, approximately 20 times more data points are required for a linear fit to be correct to one decimal place. 
\begin{figure}[h!]
\centering
    \includegraphics[scale=0.4]{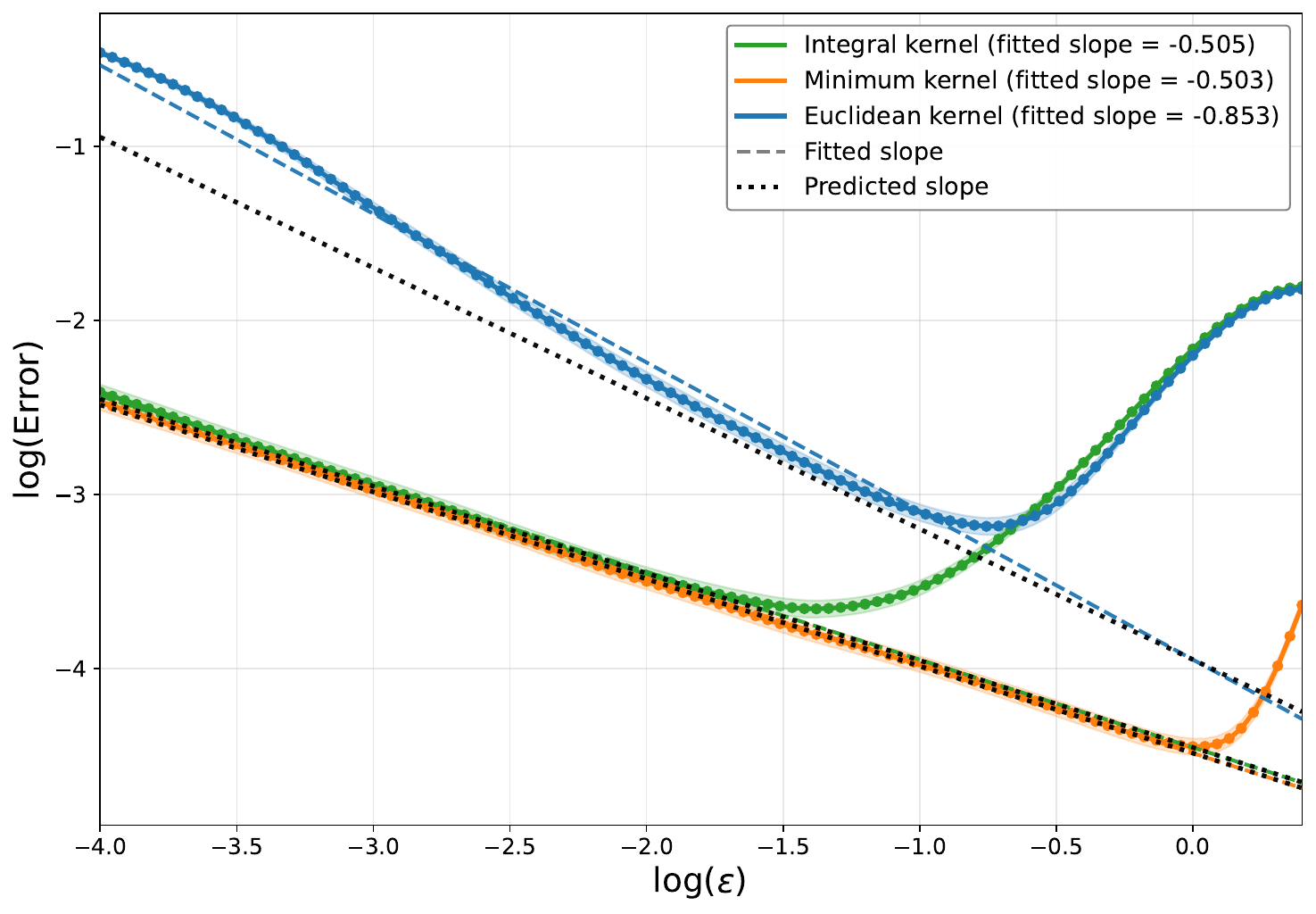}
    \caption{Error between $\frac{4}{\varepsilon} (L_{RW}f)(I_3)$ and $1$ versus $\varepsilon$ for the $G$-invariant function $f(R) = \cos\beta(R)$ using the Euclidean, minimum, and integral kernels. Shaded regions show the 95\% confidence intervals across 1000 trials.
    }
    \label{fig:so3_singlepoint_results}
\end{figure}

\section{Proof of Theorem~\ref{thm: convergence of G-invariant graph Laplacian}}
\label{sec: Proof of main thm}
We split the proof of Theorem~\ref{thm: convergence of G-invariant graph Laplacian} into three parts, corresponding to the three G-invariant kernels considered in the statement of the theorem: the minimum kernel (Section~\ref{sec: proof min kernel}), the integral kernel (Section~\ref{sec: proof integral kernel}), and the invariant features kernel (Section~\ref{sec: proof invariant features kernel}).
We note that the implicit constants in the $O$ notation below depend on $\mathcal{M}, G, \mathcal{N}$, $h$ and $x_i$ but not on $n$ or $\varepsilon$.

\subsection{Minimum Kernel}\label{sec: proof min kernel}
Fix a sufficiently small $\varepsilon > 0$ and a data point $\x_i \in \M$. Let $F(y) := K_{\min} (\x_i, y) f(y)$ and $G(y) := K_{\min} (\x_i, y)$ for all $y \in \M$.
We prove the desired result in three steps. First, we derive an asymptotic approximation for $\mathbb{E}[F]/\mathbb{E}[G]$ up to order $O(\varepsilon)$. Second, we use Bernstein's inequality to show that $\sum_{j \neq i} F(\x_j) / \sum_{j \neq i} G(\x_j)$ converges to $\mathbb{E}[F] / \mathbb{E}[G]$ at rate $O_P\left({n^{-1/2}\varepsilon^{1/2 - (d-p)/4}}\right)$.  Third, we combine steps 1 and 2  to obtain the stated error bound on $L_{RW}$. \\

\noindent {\textbf{Step 1}.}\\
 To compute the expectations of $F$ and $G$, we convert the expectations over $\M$ into integrals over the quotient $\N$ by using the following Fubini-type integration formula. This is a special case of a classical result for Riemannian submersions~\cite[Thm. 5.6]{Sakai1996}, which we adapt to our setting in Appendix~\ref{app: proof of Fubini formula}.

\begin{lemma}[Fubini-like quotient integration formula]\label{lem: Fubini formula for integration over M}
Let $(\M, g_\M)$ be a Riemannian manifold and let $G$ be a compact Lie group acting smoothly, freely, and isometrically on $\M$. Let $(\N, g_\N)$ denote the  Riemannian quotient manifold arising from the action of $G$ on $\M$. Let $\overline{\delta}$ denote the  function on $\N$ induced by the density function $\delta$ given in Definition \ref{def: density function}. Then given any integrable function $f \colon \M \to \RR$, we have that
\begin{align} \label{eq: Fubini integral on M}
    \int_\M f(x)dV_\M(x)
    =
    \int_{\N}\left( \int_{G} f(\alpha \cdot x) d\eta(\alpha) \right)\overline{\delta}([x])dV_\N([x]),
\end{align}
where $dV_\M$ and $dV_\N$ denote the Riemannian measures on $\M$ and $\N$ determined by the metrics $g_\M$ and $g_\N$, respectively, and $d\eta$ is the Haar measure on $G$ such that $\eta(G) = 1$. 
If, in addition, the function $f\colon \M \to \RR$ is G-invariant, this identity simplifies to
\begin{align} \label{eq: Fubini integral on M for G-invariant function}
    \int_\M f(x)dV_\M(x)
    =
    \int_{\N} \overline{f}([ x])  \overline{\delta}([x])dV_\N([x]).
\end{align}
\end{lemma}
We now consider local smoothing with the minimum kernel.
Due to symmetry, $K_\mathrm{min}$ induces a kernel $\overline{K}_{\mathrm{min}}$ on the quotient manifold $\N$ given by
\begin{align}
    \overline{K}_{\mathrm{min}}([x],[y]) 
    =
    \exp
    \left(
        -\min_{\alpha \in G} \| x - \alpha \cdot y \|^2/\varepsilon
    \right), \quad \ [x],[y]\in\N.
\end{align}
The minimum value over $G$ is continuous since $G$ is a compact Lie group, so for a fixed $x_i$ the minimum kernel $K_{\mathrm{min}}(x_i, y)$ is a continuous function on $\M$ and thus integrable by the compactness of $\M$. 
By Eq.~\eqref{eq: Fubini integral on M for G-invariant function} we have that
\begin{align}\label{eq: integration of f against min kernel}
\begin{split}
    \mathbb{E}\left[F\right] 
    &= \frac{1}{\vol(\M)} \int_\M K_{\min}(\x_i, y) f(y)dV_\M(y) = \frac{1}{\vol(\M)} \int_\N \overline{K}_\mathrm{min}([\x_i], [y]) \overline{f}([y])\overline{\delta}([y])dV_\N([y]). 
\end{split}
\end{align}
The following lemma, proved in Appendix~\ref{app:integration wrt minimum kernel}, is needed to evaluate the integral in Equation~\eqref{eq: integration of f against min kernel}. 
\begin{lemma}[Asymptotics of local smoothing with the minimum kernel]\label{lem: integration wrt minimum kernel}
Let $q = \dim(\N)$. Given a smooth function $h \colon \N \to \RR$ and a fixed $[x] \in \N$, for all sufficiently small $\varepsilon$ we have that
    \begin{align}
    \frac{1}{(\pi\varepsilon)^{q /2}} \int_\N \overline{K}_{\min}([x], [y]) h([y])dV_\N([y] )= h([x]) + \frac{\varepsilon}{4}\left(E([x])h([x]) - \Delta_\N h([x])\right) + O(\varepsilon^2),
\end{align}
where $dV_\N$ denotes the Riemannian measure on $\N$ determined by the metric $g_\N$, and $E([x])$ is some function on $\N$ that depends on the curvature of $\N$ at $[x]$ and on the embedding of $\M$ into $\RR^D$. 
\end{lemma}
The explicit form of $E$ is not needed later; only the fact that it is a scalar function multiplying $h$.
Under the assumption that both $\delta:\M \to \RR$ and $f:\M \to \RR$ are smooth it follows that the projections $\overline{\delta}, \overline{f}:\N \to \RR$ are both smooth by the characteristic property of surjective smooth submersions [Theorem 4.29, \cite{lee2003introduction}].
Using Lemma~\ref{lem: integration wrt minimum kernel} with $h = \overline{f}\overline{\delta}$ gives
\begin{align}\label{eq: E[F]}
\mathbb{E}\left[F\right] = \frac{(\pi\varepsilon)^{q/2}}{\vol(\M)}
\left[ \overline{f}([\x_i])\overline{\delta}([\x_i]) + \frac{\varepsilon}{4}\left( E([\x_i])\overline{f}([\x_i])\overline{\delta}([\x_i]) - \Delta_\N (\overline{f}~\overline{\delta})([\x_i]) \right) + O(\varepsilon^2) \right].
\end{align}
By taking $f =1$, we obtain
\begin{align} \label{E[G]}
\mathbb{E}\left[G\right] = \frac{(\pi\varepsilon)^{q/2}}{\vol(\M)}
\left[ \overline{\delta}([\x_i]) + \frac{\varepsilon}{4}\left( E([\x_i])\overline{\delta}([\x_i]) - \Delta_\N \overline{\delta}([\x_i]) \right) + O(\varepsilon^2) \right].\end{align}
Using the Taylor expansion $1/ (1 + a\varepsilon) = 1 - a\varepsilon + O(\varepsilon^2)$, a straightforward calculation shows that the ratio is
\begin{align}\label{eq: ratio EF/EG}
\begin{split}
    \frac{\mathbb{E}\left[F\right]}{\mathbb{E}\left[G\right]} 
    &= \overline{f}([\x_i]) + \frac{\varepsilon}{4} \left[ \frac{\Delta_\N \overline{\delta}([\x_i])}{\overline{\delta}([\x_i])} \overline{f}([\x_i]) - \frac{\Delta_\N(\overline{f}~\overline{\delta})([\x_i])}{\overline{\delta}([\x_i])} \right]+ O(\varepsilon^2) \\
    &= \overline{f}([\x_i]) + \frac{\varepsilon}{4} \left[ -\Delta_\N \overline{f}([\x_i]) + \frac{2}{\overline{\delta}([\x_i])}\langle \nabla_\N \overline{\delta}([\x_i]), \nabla_\N \overline{f}([\x_i])\rangle \right]+ O(\varepsilon^2),
 \end{split}   
\end{align}
where the second line follows from applying the product rule for the Laplace--Beltrami operator: $\Delta_\N ( \overline{f}\,\overline{\delta}) = \overline{f} \Delta_\N \overline{\delta}  - 2\langle \nabla_\N \overline{f}, \nabla_\N \overline{\delta} \rangle + \overline{\delta} \Delta_\N \overline{f}$. Note also that $\overline{\delta}([x]) > 0$ for all $[x] \in \N$ since the action of $G$ on $\M$ is free, so all orbits have a positive volume.

\vspace{3mm}
\noindent \textbf{Step 2.}\\
Let $p_+(n, \alpha)$ and $p_-(n, \alpha)$ denote the probability of having an error greater than $\alpha$ or less than $-\alpha$, respectively, i.e., 
\begin{align}\label{eq: p_+N, alpha}
    p_+ (n, \alpha) = \mathbb{P}\left[  \frac{ \sum_{j \neq i} F(\x_j) }{\sum_{j \neq i}  G(\x_j) } - \frac{\mathbb{E}[F]}{\mathbb{E}[G]} >   \alpha \right], 
    \quad 
    p_- (n, \alpha) = \mathbb{P}\left[  \frac{ \sum_{j \neq i}  F(\x_j) }{\sum_{j \neq i} G(\x_j) } - \frac{\mathbb{E}[F]}{\mathbb{E}[G]} <  -\alpha \right].
\end{align}
Define $Y_j = \mathbb{E}[G]F(\x_j) - \mathbb{E}[F]G(\x_j) + \alpha \mathbb{E}[G](\mathbb{E}[G] - G(\x_j))$ for all $j \neq i$. Multiplying both sides of Eq.~\eqref{eq: p_+N, alpha} by $\mathbb{E}[G] \sum_{j\neq i} G(\x_j)$ and rearranging gives 
\begin{align}
    p_+ (n, \alpha) =  \mathbb{P}\left[ \sum_{j \neq i} Y_j >  (n-1)\alpha (\mathbb{E}[G])^2 \right].
\end{align}
Note that $Y_j$ are i.i.d. random variables with zero mean and variance given by
\begin{align}
    \mathbb{E}[Y_j^2] 
    = \mathbb{E}[G]^2\mathbb{E}[F^2] - 2\mathbb{E}[F]\mathbb{E}[G]\mathbb{E}[&FG]  +   \mathbb{E}[F]^2\mathbb{E}[G^2] + O(\alpha). \numberthis \label{eq: expectation of Yj}
\end{align}
Next we use Lemma~\ref{lem: Fubini formula for integration over M} and Lemma~\ref{lem: integration wrt minimum kernel} to compute the second moments of $F$ and $G$. We use the fact that $\overline{K}_{\mathrm{min}}^2([x_i], [y]; \varepsilon) = \overline{K}_{\mathrm{min}}([x_i], [y]; \varepsilon/2)$ for the integral in Lemma~\ref{lem: integration wrt minimum kernel}. We obtain
\begin{align}\label{eq: E[F^2]}
    \mathbb{E}[F^2] = \frac{(\pi\varepsilon)^{q/2}}{2^{q/2}\vol(\M)}
\left[ \overline{f}^2([\x_i])\overline{\delta}([\x_i]) + \frac{\varepsilon}{8}\left( E([\x_i])\overline{f}^2([\x_i])\overline{\delta}([\x_i]) - \Delta_\N (\overline{f}^2~\overline{\delta})([\x_i]) \right) + O(\varepsilon^2) \right],
\end{align}

\begin{align}\label{eq:E[G^2]}
\mathbb{E}\left[G^2\right] = \frac{(\pi\varepsilon)^{q/2}}{2^{q/2}\vol(\M)}
\left[ \overline{\delta}([\x_i]) + \frac{\varepsilon}{8}\left( E([\x_i])\overline{\delta}([\x_i]) - \Delta_\N \overline{\delta}([\x_i]) \right) + O(\varepsilon^2) \right],\end{align}
and
\begin{align}\label{E[FG]}
    \mathbb{E}[FG] = \frac{(\pi\varepsilon)^{q/2}}{2^{q/2}\vol(\M)}
\left[ \overline{f}([\x_i])\overline{\delta}([\x_i]) + \frac{\varepsilon}{8}\left( E([\x_i])\overline{f}([\x_i])\overline{\delta}([\x_i]) - \Delta_\N (\overline{f}~\overline{\delta})([\x_i]) \right) + O(\varepsilon^2) \right].
\end{align}
Substituting Eqs.~\eqref{eq: E[F]},~\eqref{E[G]},~\eqref{eq: E[F^2]},~\eqref{eq:E[G^2]}, and~\eqref{E[FG]} into Eq.~\eqref{eq: expectation of Yj} gives
\begin{align}
    \mathbb{E}[Y_j^2] = 
\frac{(\pi\varepsilon)^{3q/2}\overline{\delta}^3([\x_i])}{2^{q/2}\vol(\M)^3} \frac{\varepsilon}{8} \left[ -\Delta_\N \overline{f}^2([\x_i]) + 2\overline{f}([\x_i]) \Delta_\N \overline{f}([\x_i]) + O(\varepsilon) \right] + O(\alpha), 
\end{align}
and after applying the identity $-\Delta_\N \overline{f}^2([x_i]) + 2\overline{f}([x_i]) \Delta_\N \overline{f}([x_i]) = 2 \| \nabla_\N \overline{f}([x_i]) \|^2$, this simplifies to
\begin{align}
    \mathbb{E}[Y_j^2] =
\frac{(\pi\varepsilon)^{3q/2}\overline{\delta}^3([x_i])}{2^{q/2}\vol(\M)^3} \frac{\varepsilon}{4} \left[ \|\nabla_\N \overline{f}([x_i])\|^2 + O(\varepsilon) \right] + O(\alpha).    
\end{align}
We have that $F$ and $G$ are bounded random variables since $\M$ is compact, so by construction the random variables $Y_j$ are also bounded, say $|Y_j| < 3C$ for all $j \neq i$ for some $C > 0$. It follows from Bernstein's inequality that
\begin{align}\label{eq: p(n, alpha) after Bernstein}
\begin{split}
    p_+(n, \alpha) 
    &\leq \exp\left( -\frac{(n-1)^2\alpha^2\mathbb{E}[G]^4}{2(n-1)\mathbb{E}[Y_j^2] + 2C (n-1)\alpha \mathbb{E}[G]^2  }\right) \\
    &=\exp\left( - \frac{(n-1)\alpha^2 \bar{\delta}([x_i])2^{q/2 + 1} (\pi\varepsilon)^{q/2}}{\vol(\M)\varepsilon(\|\nabla_\N \overline{f}([x_i])\|^2 + O(\varepsilon) ) + O(\alpha)}\right). 
    \end{split}
\end{align}
 Assume $\nabla_\N \overline{f}([x_i]) \neq 0$; the case $\nabla_\N \overline{f}([x_i]) = 0$ is covered in Remark~\ref{rmk: improved convergence}. Fix a constant $M >0$ and set $\alpha = M n^{-1/2}\varepsilon^{1/2 - q/4}$. Then Eq.~\eqref{eq: p(n, alpha) after Bernstein} becomes
\begin{align*}
    p_+(n, \alpha) 
\leq \exp\left( -\frac{C_0 M^2}{1 +  C_1 M n^{-1/2}\varepsilon^{-1/2 - q/4} }\right) \numberthis
\end{align*}
for appropriate constants $C_0$ and $C_1$. 
We obtain the same bound for $p_-(n,\alpha)$ by changing the sign of $\alpha$ in the definition of $Y_j$, which does not affect the approximation of $\mathbb{E}[Y_j^2]$ in Eq.~\eqref{eq: expectation of Yj}.
Note that there exists a positive integer $N$ such that $C_1 n^{-1/2}\varepsilon^{-1/2 - q/4} < 1$ for all $n > N$, and in this case $p_\pm(n, \alpha) \leq \exp( - C_0 M^2/(1+M))$.
Thus for any $\delta > 0$ we can choose sufficiently large constants $M$ and $N$ such that  $p_{\pm}(n, \alpha)< \delta$ for all $n > N$. Recalling that $\alpha = M n^{-1/2}\varepsilon^{1/2 - q/4}$ and the definition of $p_{\pm}(n,a)$ in Eq.~\eqref{eq: p_+N, alpha} gives the error 
\begin{align}\label{eq: error bound on EF/EG - sums}
   \Bigg\vert\frac{ \sum_{j \neq i} F(\x_j) }{\sum_{j \neq i}G(\x_j) } - \frac{\mathbb{E}[F]}{\mathbb{E}[G]}\Bigg\vert = O_P\left(\frac{1}{n^{1/2}\varepsilon^{-1/2+q/4}}\right).
\end{align}

\vspace{3mm}
\noindent \textbf{Step 3.}\\
With the notation $F(y) = K_{\min} (\x_i, y) f(y)$ and $G(y) = K_{\min} (\x_i, y)$, $y \in \M$, we can rewrite Eq.~\eqref{eq: RW graph Laplacian acting on functions} as 
\begin{align}\label{eq: graph Laplacian with min kernel}
    \frac{4}{\varepsilon }L_{RW}f(\x_i)
    =\frac{4}{\varepsilon}\left[ f(\x_i) - \frac{ \sum_{j=1}^n F(\x_j)}{\sum_{j=1}^n G(\x_j)}\right].
\end{align}
As shown by \cite{SteerablePaper}, removing the diagonal terms from the sums introduces an $O\left({1}/{(n \varepsilon^{(d - p)/2}})\right)$ error, which is negligible compared to the variance error term in Eq.~\eqref{eq: error bound on EF/EG - sums}. Then from Eqs.~\eqref{eq: ratio EF/EG},~\eqref{eq: error bound on EF/EG - sums} and~\eqref{eq: graph Laplacian with min kernel} we see that
\begin{align}
\begin{split}
    &\Bigg\vert \frac{4}{\varepsilon} L_{RW}f(x_i) - \Big(\Delta_\N \overline{f}([\x_i]) - 2\langle \nabla_\N \log \overline{\delta}([\x_i]), \nabla_\N \overline{f}([\x_i])\rangle + O(\varepsilon)\Big)\Bigg\vert \\
  &=\Bigg\vert \frac{4}{\varepsilon}\left( \bar f([\x_i]) - \frac{ \sum_{j\neq i}F(\x_j) }{\sum_{j\neq i} G(\x_j) } \right) + \frac{4}{\varepsilon} \left( \frac{\mathbb{E}[F]}{\mathbb{E}[G]} - \bar{f}([x_i]) \right) \Bigg\vert = O_P\left(\frac{1}{n^{1/2}\varepsilon^{1/2 + q/4 }}\right). 
\end{split}
\end{align}
Recalling that $q = \mathrm{dim}(\N) = d - p$ gives the desired result.

\subsection{Integral Kernel}\label{sec: proof integral kernel}
Our proof for the integral kernel is based on the results of \citet{ROSEN2024}, which constructed a $G$-invariant graph Laplacian (G-GL) by considering the distances between
all the pairs of points generated by the action of $G$ on the dataset, followed by integration over $G$. The key result is their Theorem 11, which shows that the normalized $G$-GL converges to the Laplace--Beltrami operator $\Delta_\M$ on $\M$ with an improved variance error of $O_P(1/(n^{1/2}\varepsilon^{1/2 + (d - p)/4 }))$. We then use the projection of $\Delta_\M$ onto the quotient manifold $\N$ to obtain a differential operator on $\N$, as given in Lemma~\ref{lem: radial part of the Laplace-Beltrami operator}.

Given data points $\x_1, \dots, \x_n \in \M$, let $\tilde{L}$ denote the normalized $G$-GL Laplacian of \citet{ROSEN2024}, which acts on functions $g \colon \{1, \dots, n\}\times G \to \RR$ by
\begin{align}\label{eq: G-invariant graph Laplacian}
   \tilde{L} g(i, \alpha) 
   =  g(i, \alpha) - \frac{\sum_{j=1}^n \int_G \exp\left(-\frac{\| \alpha\cdot\x_i - \beta \cdot \x_j \|^2}{\varepsilon}\right) g(j, \beta) d\eta(\beta)}{\sum_{j=1}^n \int_G \exp\left(-\frac{\| \x_i - \beta \cdot \x_j \|^2}{\varepsilon}\right) d\eta(\beta)}.
\end{align}
Now, given a $G$-invariant smooth function $f \colon \M \to \RR$, we obtain a function 
$g \colon \{1, \dots, n\}\times G \to \RR$ by defining $g(i, \alpha) = f(\alpha \cdot \x_i)$. Note that $g$ is independent of $G$, i.e. $g(i, \alpha) = g(i, e)$ for all $\alpha \in G$, due to the $G$-invariance of $f$. It follows immediately from the above definitions that $(L_{RW}f)(\x_i)$ coincides with $(\tilde{L}g)(i, e)$ for all $\x_1, \dots, \x_n$.
Therefore, by the convergence of the $G$-GL operator $\tilde{L}$  \cite[Thm 11]{ROSEN2024}, we obtain
\begin{equation}
    \frac{4}{\varepsilon} (L_{RW} f) (\x_i) = \frac{4}{\varepsilon} (\tilde{L}g)(i,e) = \Delta_\M f(\x_i) + O(\varepsilon) + O_P\left(\frac{1}{n^{1/2}\varepsilon^{1/2 + (d - p)/4 }}\right).
\end{equation}
Since $f$ is $G$-invariant, we have that $f = \pi^\ast \bar{f}$ with $\bar f$ the induced function on $\N$, so $\Delta_\M f(x_i) = P(\Delta_\M \bar{f})([x_i])$ by the definition of $P(\Delta_\M)$ in Eq.~\eqref{eq: def of projection of Laplace-Beltrami operator}. Applying  Lemma~\ref{lem: radial part of the Laplace-Beltrami operator} then gives
\begin{align}
    \Delta_\M f(\x_i) = \Delta_\N \overline{f} ([\x_i]) - \langle \nabla_\N \log \overline{\delta}([\x_i]), \nabla_\N \overline{f} ([\x_i]) \rangle.
\end{align}
Hence,
\begin{align}
    \frac{4}{\varepsilon}
    (L_{RW} f)(\x_i)
    =
    \Delta_\N \overline{f}([\x_i])
    -
    \langle \nabla_\N \log \overline{\delta}([\x_i]), \nabla_\N \overline{f}([\x_i]) \rangle+ O(\varepsilon)
    +
    O_P\left(
        \frac{1}{n^{1/2}\varepsilon^{1/2 + (d - p)/4 }}
    \right),
\end{align}
as we wished to prove.

\subsection{Invariant Features Kernel}\label{sec: proof invariant features kernel}
Let $\phi$ be the $G$-invariant map as defined in Eq.~\eqref{eq: invariant feature kernel}. By assumption, $\mathrm{im}\phi$ is a submanifold of $\RR^E$ diffeomorphic to $\N$ and therefore it has dimension $q = d - p$. Given data points $\x_1, \dots, \x_n$ sampled from the uniform measure $dV_\M$ on $\M$, then $\phi(\x_1), \dots, \phi(\x_n)$ are sampled from the pushforward measure $\phi_\ast(dV_\M)$ on $\mathrm{im}\phi$. Let $p_\phi$ denote the PDF of $\phi_\ast (dV_\M)$.
Let $\tilde{L}_{RW}$ denote the normalized graph Laplacian constructed from the standard Gaussian kernel on $\mathrm{im}(\phi)$ and let $\bar{\phi}$ be the projection of $\phi$ onto $\N$. Consider the smooth function $\overline{f} \circ \bar{\phi}^{-1} \colon \mathrm{im}\phi \to \RR$ and note that $f(x) = \overline{f}\circ \bar{\phi}^{-1}(\phi(x))$ for all $x \in \M$, so it follows that $L_{RW} f(\x_i) = \tilde{L}_{RW} \overline{f} \circ \bar{\phi} ^{-1} (\phi(\x_i))$ for all data points $\x_i$.
Denoting $\y_i := \phi(\x_i)$, Theorem \ref{thm: convergence of standard graph Laplacian to Fokker-Planck operator} implies
\begin{align}
\begin{split}
    &\frac{4}{\varepsilon}
    L_{RW} f(\x_i)
    =
    \frac{4}{\varepsilon} \widetilde{L}_{RW} \overline{f} \circ \bar{\phi} ^{-1} (\y_i) \\
    &=
    \Delta_{\mathrm{im}\phi}\overline{f} \circ \bar{\phi}^{-1} (\y_i)
    -
    2\,\bigl\langle \nabla_{\mathrm{im}\phi} \log p_\phi\left(\y_i\right), \nabla_{\mathrm{im}\phi} \overline{f} \circ \bar{\phi} ^{-1} (\y_i) \bigr\rangle
    +
    O_P\!\left(\varepsilon + \frac{1}{n^{1/2}\,\varepsilon^{1/2+(d-p)/4}}\right).
\end{split}
\label{eq:LN-expansion}
\end{align}
\noindent To conclude the proof, let us write 
\begin{align}
    \D_{\mathrm{im}\phi}(\overline{f}\circ\bar{\phi}^{-1})\left(\y_i\right)
    =
    \Delta_{\mathrm{im}\phi}(\overline{f} \circ \bar{\phi}^{-1} )(\y_i) 
    -
    2\langle \nabla_{\mathrm{im}\phi}\log p_\phi(\y_i), \nabla_{\mathrm{im}\phi} (\overline{f} \circ \bar{\phi}^{-1}) (\y_i)\rangle,
\end{align}
and note that 
\begin{align}
    \D_{\mathrm{im}\phi}\left(\overline{f}\circ\bar{\phi}^{-1}\right)
    \left( \phi(\x_i) \right)
    =
    \D_{\mathrm{im}\phi} ( \overline{f} \circ \bar{\phi}^{-1})
    \left( \bar{\phi}([\x_i]) \right)
    =
    (\bar{\phi}^\ast \D_{\mathrm{im}\phi})\overline{f}([\x_i]),
\end{align}
where $\bar{\phi}^\ast \D_{\mathrm{im}\phi}$ denotes the pullback of the operator $\D_{\mathrm{im} \phi}$ to $\N$.

\section{Experiments}\label{sec:experiments}
In this section, we demonstrate our framework through several numerical experiments, performing spectral embedding with each invariant kernel and comparing to the standard Euclidean kernel $\exp(-\|x - y\|^2/\varepsilon)$. The code for reproducing these results is available at:
\url{https://github.com/yearivig/G_invariant_spectral_embedding}

\subsection{3D Point Clouds}

We obtained synthetic point clouds from the molecular structure of the Glucagon polypeptide, entry 1GCN \citep{Sasaki1975} in the Protein Data Bank (PDB) \citep{Rose2017}. 
We generated $n = 800$ point clouds by rotating the $\psi$-torsion angle of the 19th amino acid residue within the Glucagon over a full circle, see Figure~\ref{fig:overlay_glucagon}; the motion of the residue constitutes the intrinsic geometry in this experiment.

\begin{figure}[h!]
    \centering
    \begin{subfigure}[b]{0.3\textwidth}
        \centering
        \adjustimage{trim={0pt} {0.10\height} {0pt} {0.22\height}, clip, width=\textwidth}{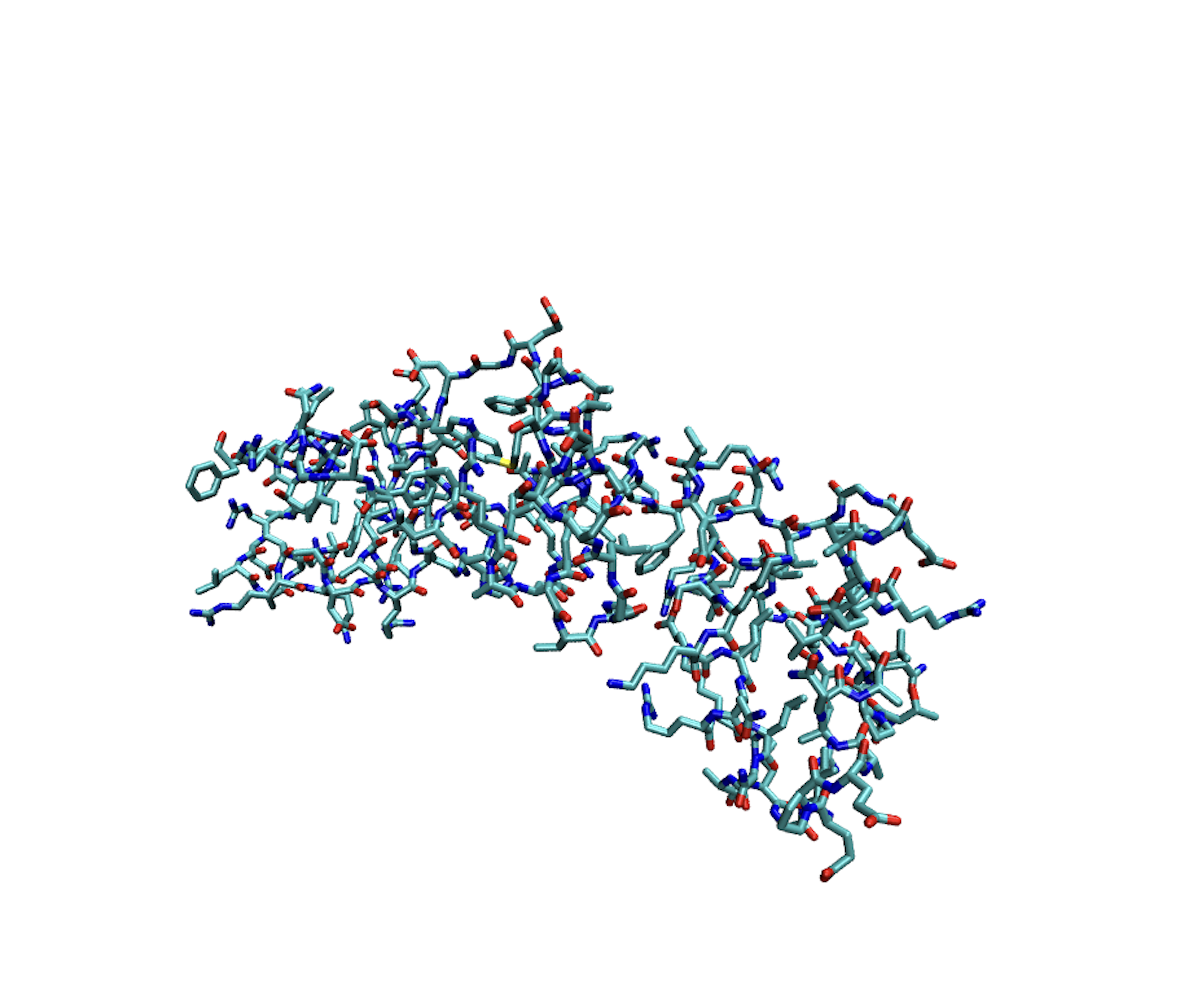}
        \caption{A single conformation.}
        \label{fig:overlay_glucagon_single}
    \end{subfigure}
    \hspace{0.03\textwidth}
    \begin{subfigure}[b]{0.3\textwidth}
        \centering
        \adjustimage{trim={0pt} {0.12\height} {0pt} {0.22\height}, clip, width=\textwidth}{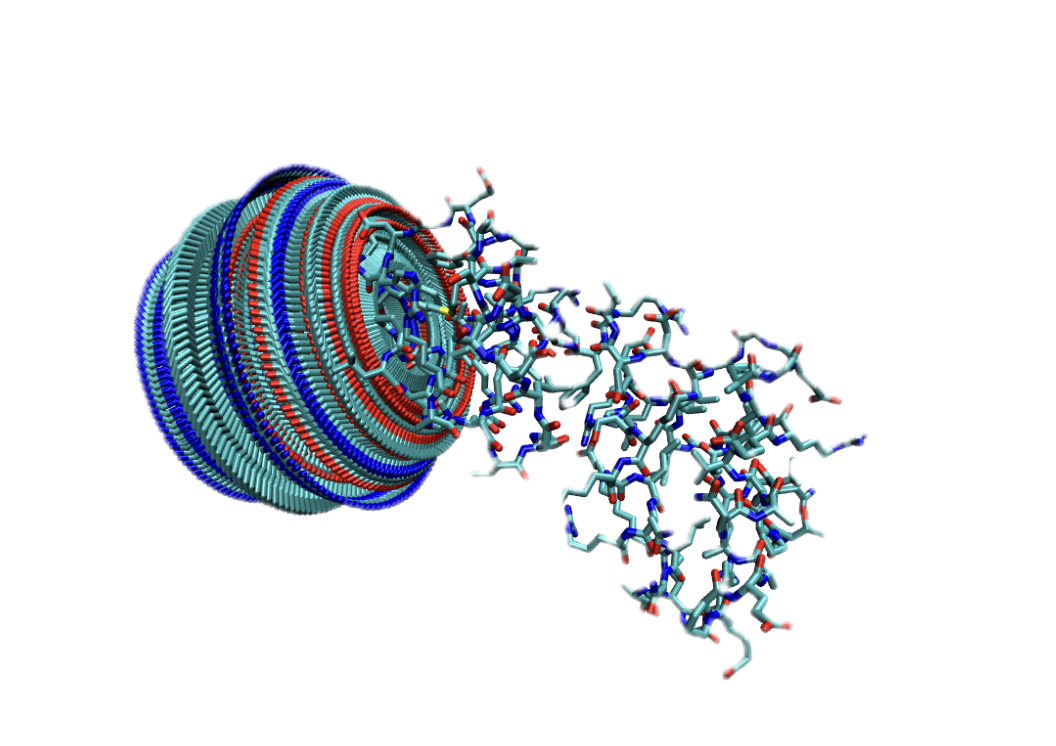}
        \caption{Overlay of all frames.}
        \label{fig:overlay_glucagon_all}
    \end{subfigure}
    \caption{The Glucagon molecule as a point cloud (the displayed sticks are for visualization purposes): a single conformation (left) and the overlay of all frames showing the full rotation movement of the $\psi$-torsion angle (right), shown at the same viewing angle. In our experiments each conformation is treated as a point cloud of the $383$ atomic coordinates.}
    \label{fig:overlay_glucagon}
\end{figure}

Each point cloud consists of 383 points in $\RR^3$, obtained by centering the 303 points belonging to the rotating part and then appending the 80 stationary ones.
We then applied a random global rotation sampled uniformly from $\mathrm{SO}(3)$, which represent nuisance parameters and extrinsic motion.  
The resulting point clouds constituted the clean dataset. The noisy dataset was created by adding, before the global rotation, i.i.d. Gaussian noise with mean zero and variance one tenth of the mean squared coordinate of each point cloud,
$\sigma_i^2 = \lVert X_i \rVert_F^2 / (3 \cdot 383 \cdot 10)$, corresponding to a signal-to-noise ratio (SNR) of 10 dB.
The point clouds lie (approximately) on a manifold $\M$ of dimension $d = 4$, and their intrinsic geometry is a circle, i.e., $\N \cong \mathbb{S}^1$.

\begin{figure}[h!]
\centering
\setlength{\tabcolsep}{3pt}
\renewcommand{\arraystretch}{1.15}
\begin{tabular}{@{}c c cccc@{}}
\toprule
Noise & $n$ & Minimum & Integral & \makecell{Invariant\\Features} & Euclidean \\
\midrule
  & 200 & \makecell{\expfig{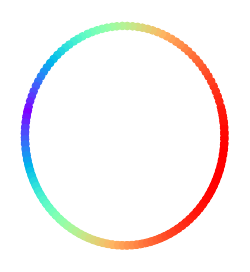}}
        & \makecell{\expfig{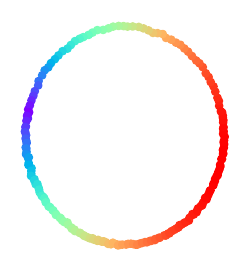}}
        & \makecell{\expfig{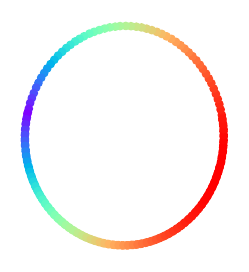}}
        & \makecell{\expfig{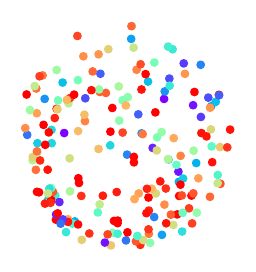}} \\
Clean & 400 & \makecell{\expfig{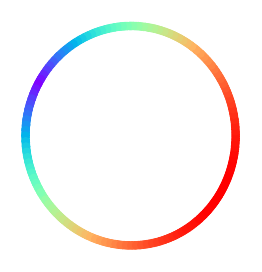}}
        & \makecell{\expfig{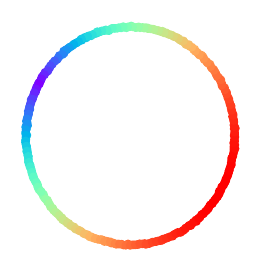}}
        & \makecell{\expfig{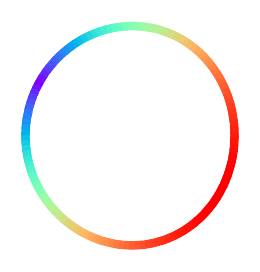}}
        & \makecell{\expfig{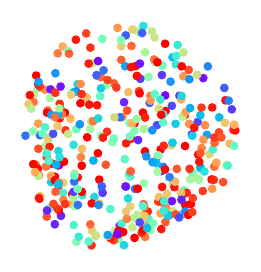}} \\
  & 800 & \makecell{\expfig{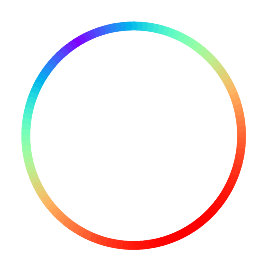}}
        & \makecell{\expfig{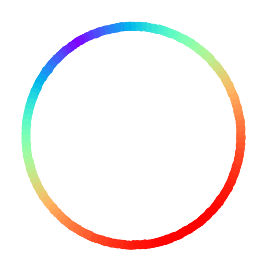}}
        & \makecell{\expfig{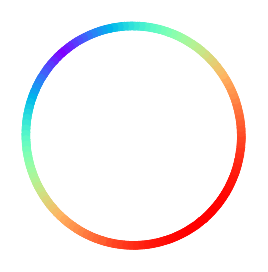}}
        & \makecell{\expfig{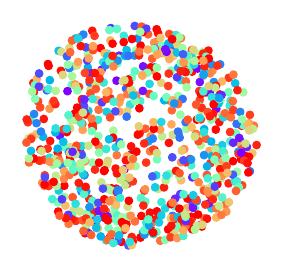}} \\
\midrule
  & 200 & \makecell{\expfig{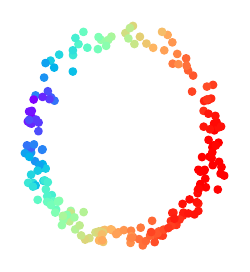}}
        & \makecell{\expfig{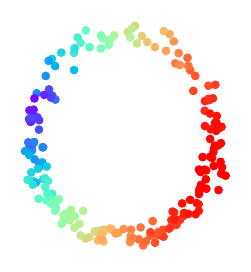}}
        & \makecell{\expfig{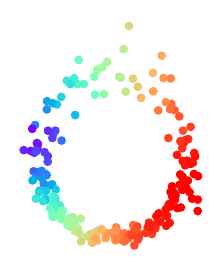}}
        & \makecell{\expfig{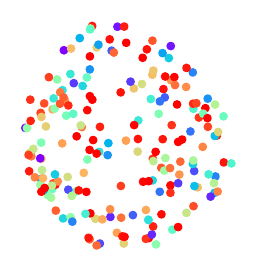}} \\
Noisy & 400 & \makecell{\expfig{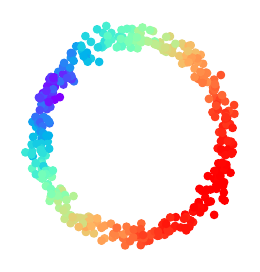}}
        & \makecell{\expfig{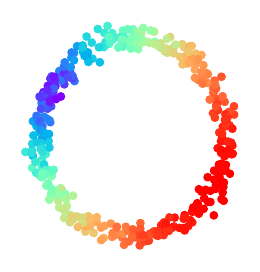}}
        & \makecell{\expfig{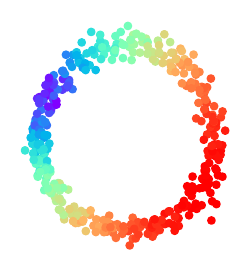}}
        & \makecell{\expfig{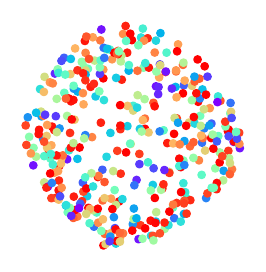}} \\
  & 800 & \makecell{\expfig{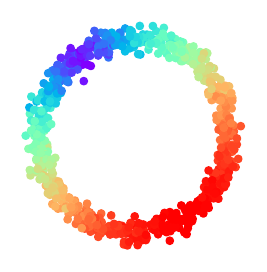}}
        & \makecell{\expfig{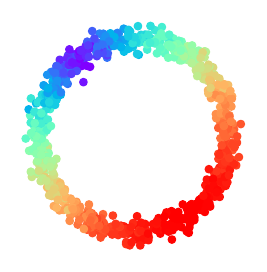}}
        & \makecell{\expfig{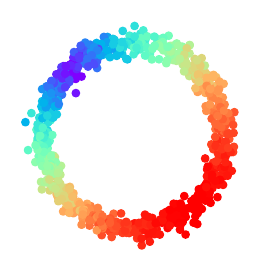}}
        & \makecell{\expfig{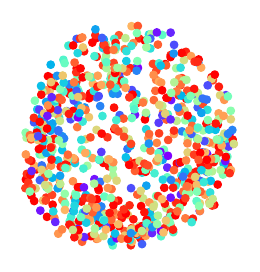}} \\
\bottomrule
\end{tabular}
\caption{Spectral embedding of clean and noisy 3D point clouds of the Glucagon molecule with rotational torsion based on the Euclidean vs.\ $\mathrm{SO}(3)$-invariant kernels for $n \in \{200, 400, 800\}$ data points.}
\label{fig: point clouds experiment}
\end{figure}

We performed two-dimensional spectral embeddings of the data using Algorithm~\ref{alg:spectral embedding} with each of the following kernels: Euclidean, minimum kernel, integral kernel, and invariant features kernel. For the three invariant kernels, the symmetry group is $G = \mathrm{SO}(3)$, acting on a point cloud $X \in \RR^{383 \times 3}$ via $XR^\top$ for a $3 \times 3$ rotation matrix $R$.
For invariant features, we use the map given by the Gram matrix\footnote{Since $\mathrm{SO}(3) \subset \mathrm{O}(3)$, the Gram matrix $XX^\top$ is in particular $\mathrm{SO}(3)$-invariant.} (Example~\ref{ex: Gram matrix} in Section~\ref{sec:G-invariant kernels}). The bandwidth parameter $\varepsilon$ was set to $\varepsilon = 47$ for the minimum and integral kernels, and $\varepsilon = 3000$ for the invariant features and Euclidean kernels. Figure~\ref{fig: point clouds experiment} shows the resulting embeddings for both the clean and noisy datasets. The key finding is that the embeddings based on each of the three $\mathrm{SO}(3)$-invariant kernels recover the correct intrinsic geometry of the data: a circle. In contrast, the embedding based on the Euclidean kernel fails to capture this circular geometry even for a large number of samples.

\subsection{Tomographic Images}
In this test, we generated 2D projection images from the 3D molecular structure of the Glucagon polypeptide (PDB ID 1GCN \citep{Sasaki1975}). Similarly to the previous experiment, we created $n = 1000$ volumes by rotating the $\psi$-torsion angle of the 19th amino acid residue in the Glucagon molecule, but in this case the motion followed a half-circle trajectory only. This constitutes the intrinsic motion in our data.  

Using the ASPIRE package \citep{https://doi.org/10.5281/zenodo.5657281}, each volume was projected to a 2D image using tomographic projection from a fixed point of view. Finally, we rotated each image by a random element in $\mathrm{SO}(2)$, which gives nuisance parameters and extrinsic motion in the data. This constituted the clean dataset. The noisy dataset was created by adding, before the global 2D rotation, i.i.d. Gaussian noise with mean zero and variance given by 0.79 times the mean squared pixel value of each image, corresponding to an SNR of 1 dB. These tomographic images lie (approximately) on a manifold $\M$ of dimension $d = 2$, and their intrinsic geometry is a half-circle. The group $G = \mathrm{SO}(2)$ acts by in-plane rotation of the 2D images.

As in the previous experiment, we performed different two-dimensional spectral embeddings of the data by applying Algorithm~\ref{alg:spectral embedding} with each of the following kernels: Euclidean, minimum, integral, and invariant features. In this case, the symmetry group is $G = \mathrm{SO}(2)$ and the invariant features map is given by the rotational bispectrum of the 2D images (see Example~\ref{ex:bispectrum} in Section~\ref{sec:G-invariant kernels}). The bandwidth parameter was set to $\varepsilon = 0.005$ for the minimum and integral kernels,  $\varepsilon = 3 \times 10^{-8}$ for the bispectrum invariant features kernel, and to $\varepsilon = 0.01$ for the  Euclidean kernel. Figure~\ref{fig:exp2_fig} shows the resulting embeddings for both the clean and noisy datasets. We observe that the embeddings based on each of the three $\mathrm{SO}(2)$-invariant kernels recover the correct intrinsic geometry of the data, a half-circle, while the embedding based on the Euclidean kernel fails to capture this geometry.

\begin{figure}[h!]
\centering
\setlength{\tabcolsep}{3pt}
\renewcommand{\arraystretch}{2}
\begin{tabular}{@{} >{\centering\arraybackslash}m{4em} ccccc@{}}
\toprule
Noise & Minimum & Integral & \makecell{Invariant\\Features}
      & \makecell{Euclidean (intrinsic\\ rotation color)}
      & \makecell{Euclidean (extrinsic\\ rotation color)} \\
\midrule
Clean  & \makecell{\expfig{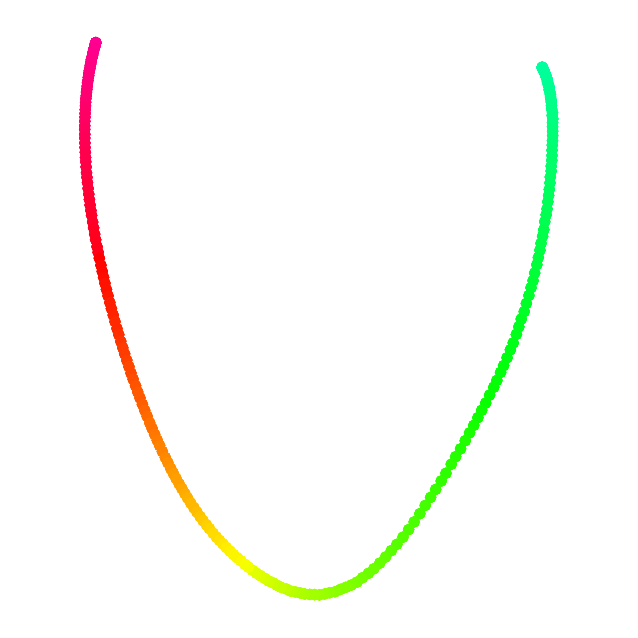}}
    & \makecell{\expfig{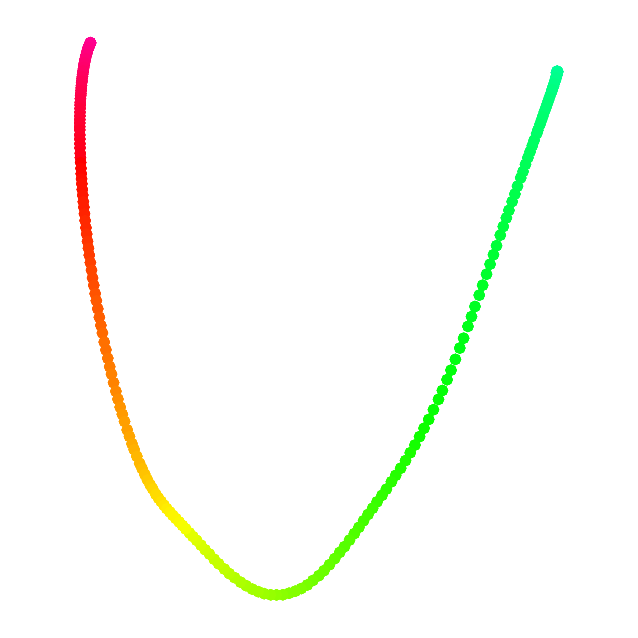}}
    & \makecell{\expfig{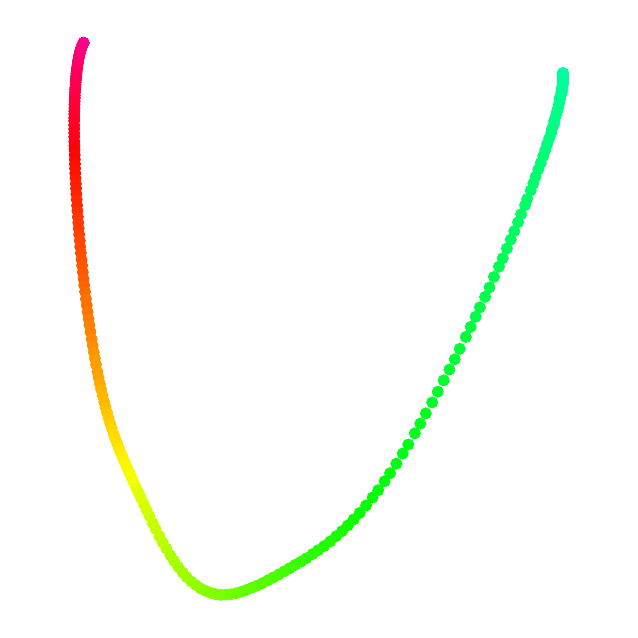}}
    & \makecell{\expfig{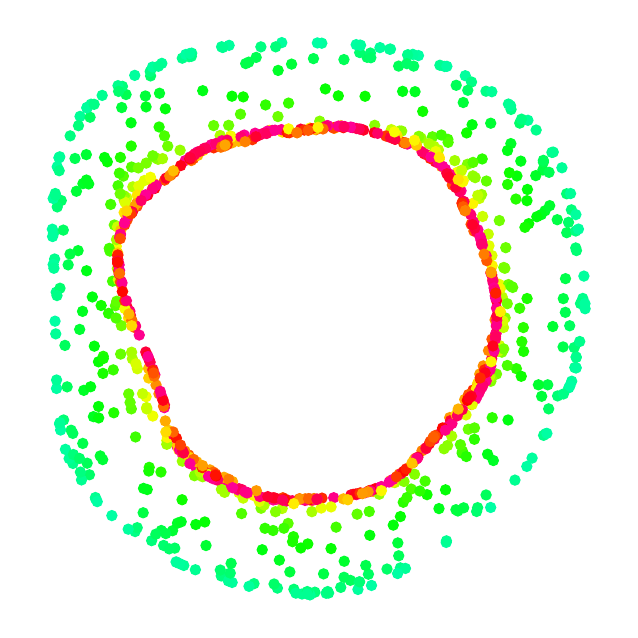}}
    & \makecell{\expfig{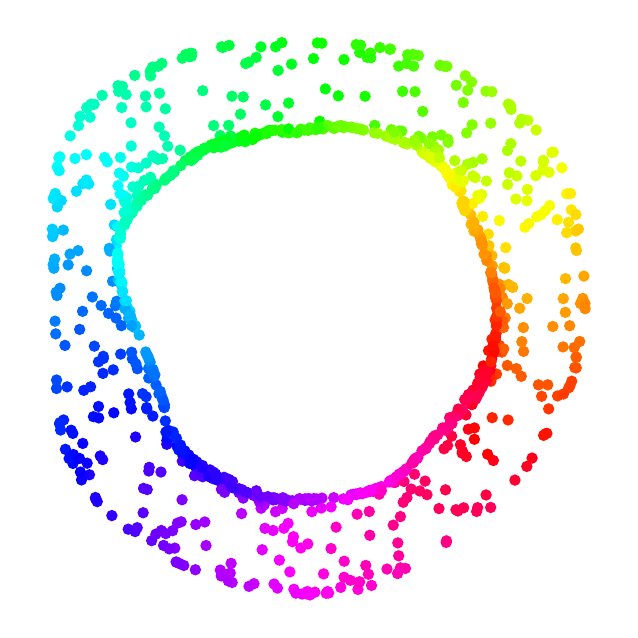}} \\
Noisy &\makecell{ \expfig{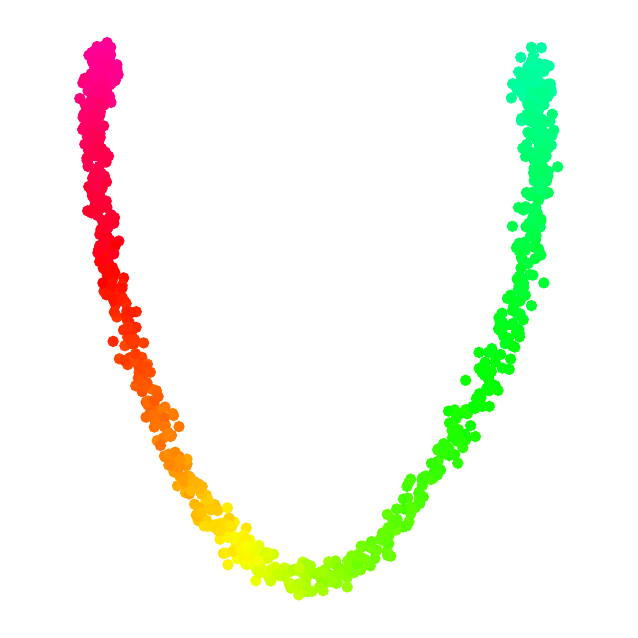}}
    &\makecell{ \expfig{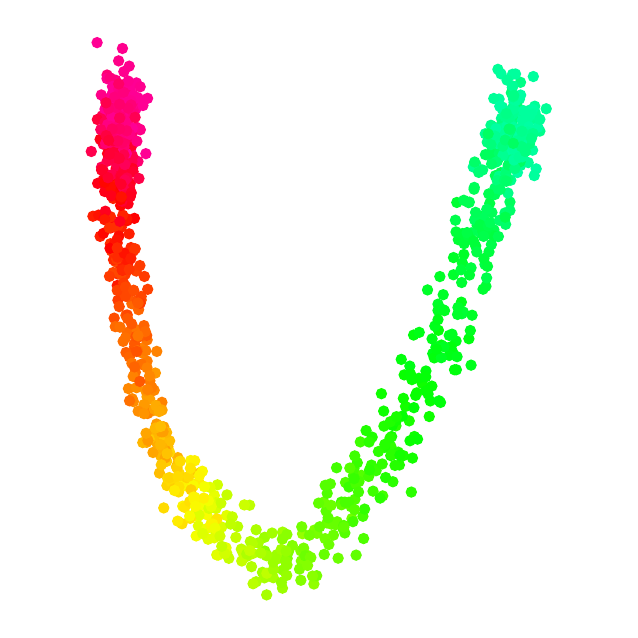}}
    & \makecell{\expfig{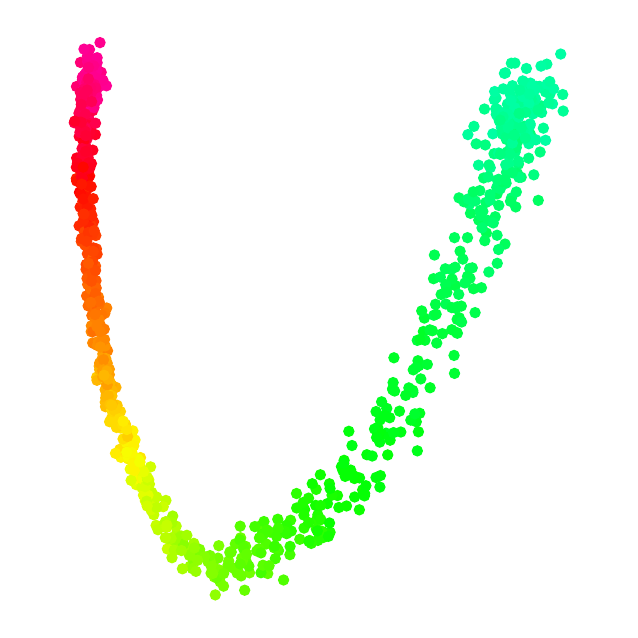}}
    &\makecell{ \expfig{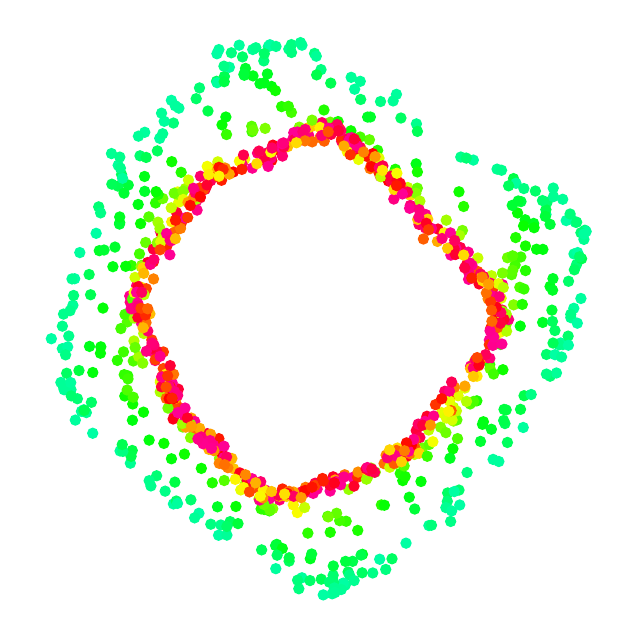}}
    &\makecell{ \expfig{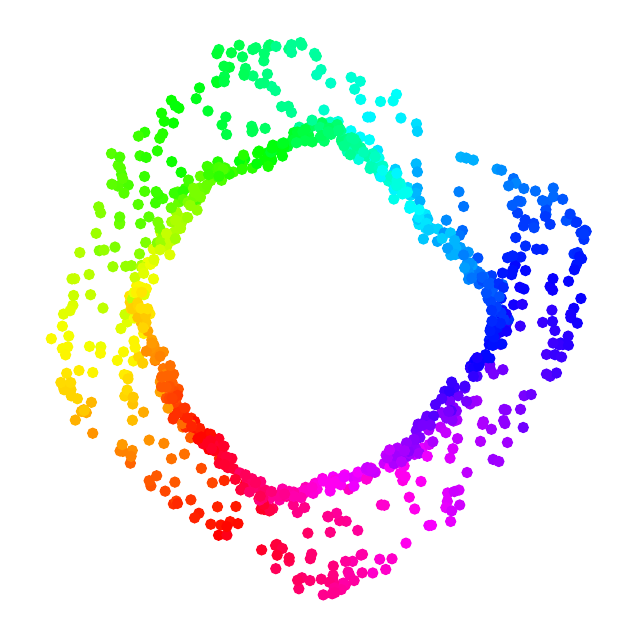} }\\
\bottomrule
\end{tabular}
\caption{Spectral embedding of clean and noisy tomographic images based on the Euclidean vs.\ $\mathrm{SO}(2)$-invariant kernels for $n=1000$ data points. The last two columns show the Euclidean kernel embedding colored according to the intrinsic and extrinsic rotation angles, respectively.}
\label{fig:exp2_fig}
\end{figure}

\subsection{Spinning Objects}
We use a dataset built from COIL-100 \citep{nene1996columbia}, which provides
72 snapshots of different 3D objects taken at $5^\circ$ intervals on a motorized turntable. We composite objects 87 and 72 side by side on a single greyscale
canvas, over all $72 \times 72 = 5184$ pairs of angles
$(\theta_A, \theta_B)$. We then rotated each 2D image by a uniformly-drawn angle. Figure~\ref{fig:spinning_toys} shows two example images from this dataset.
Given the two independent rotations of the 3D objects, the intrinsic geometry of the data is topologically equivalent to a torus $\mathbb{S}^1 \times \mathbb{S}^1$. The group $G = \mathrm{SO}(2)$ acts on each image by in-plane rotations.

\begin{figure}[h!]
    \centering
    \vspace{1em}
    \begin{subfigure}[t]{0.24\textwidth}
        \centering
        \includegraphics[width=0.9\textwidth]{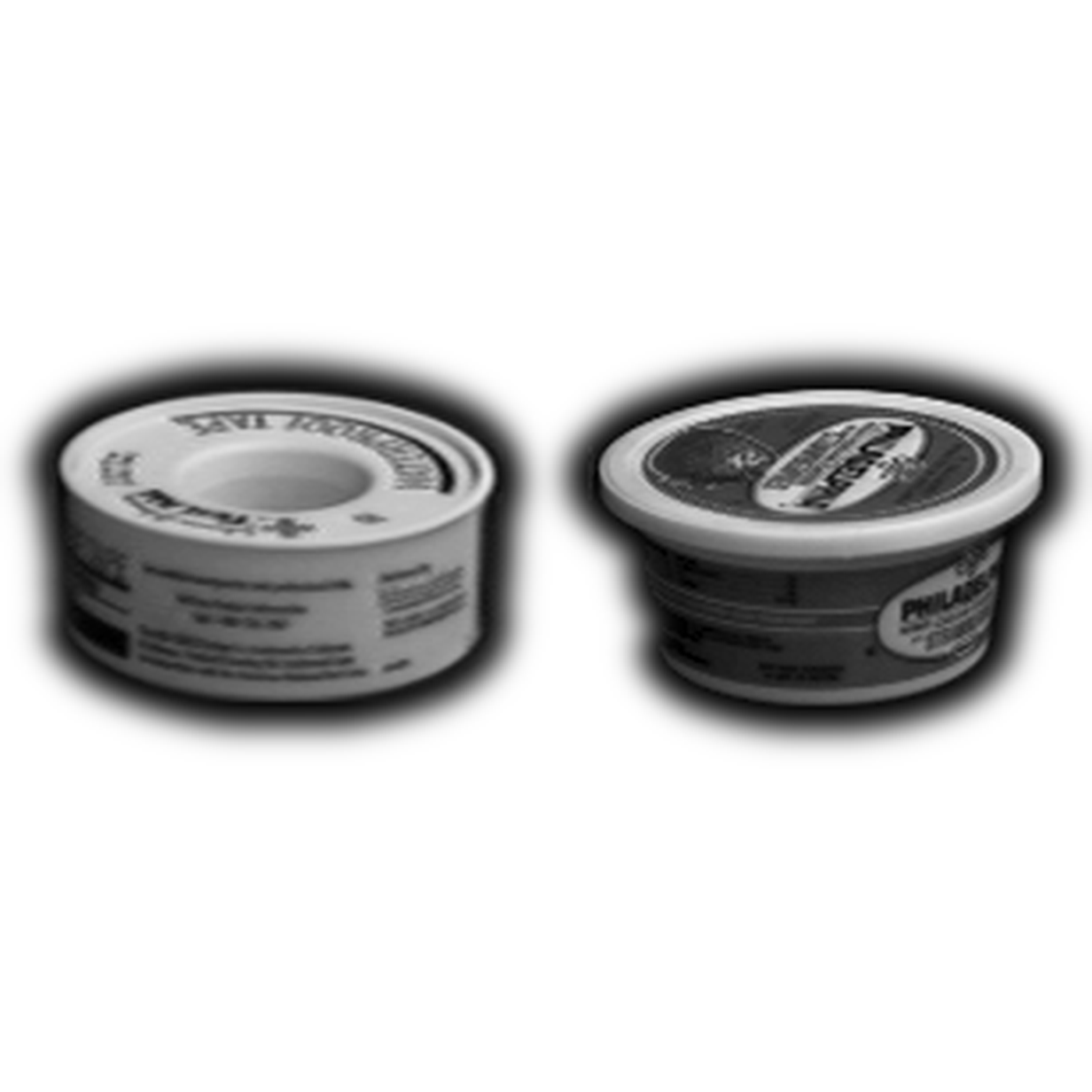}
        \vspace{0.5em}
        \caption{$\theta_A = 155^\circ, \, \theta_B = 305^\circ$ }
        \label{fig:original_imageA}
    \end{subfigure}
    \begin{subfigure}[t]{0.24\textwidth}
        \centering
        \includegraphics[width=0.9\textwidth]{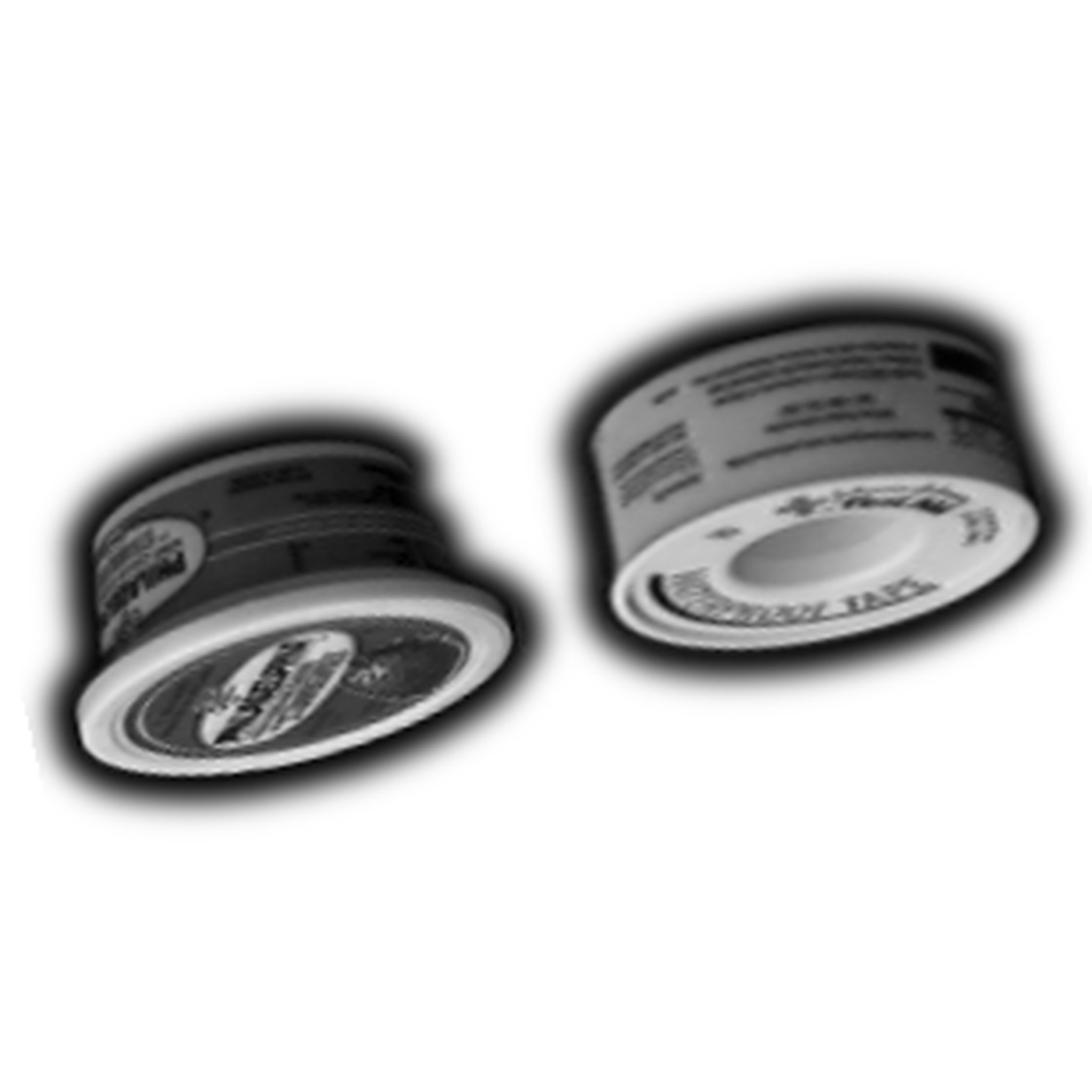}
        \vspace{0.5em}
        \caption{Outer rotation $192.8^\circ$}
        \label{fig:rotated_imageA}
    \end{subfigure}
    
    \begin{subfigure}[t]{0.24\textwidth}
        \centering
        \includegraphics[width=0.9\textwidth]{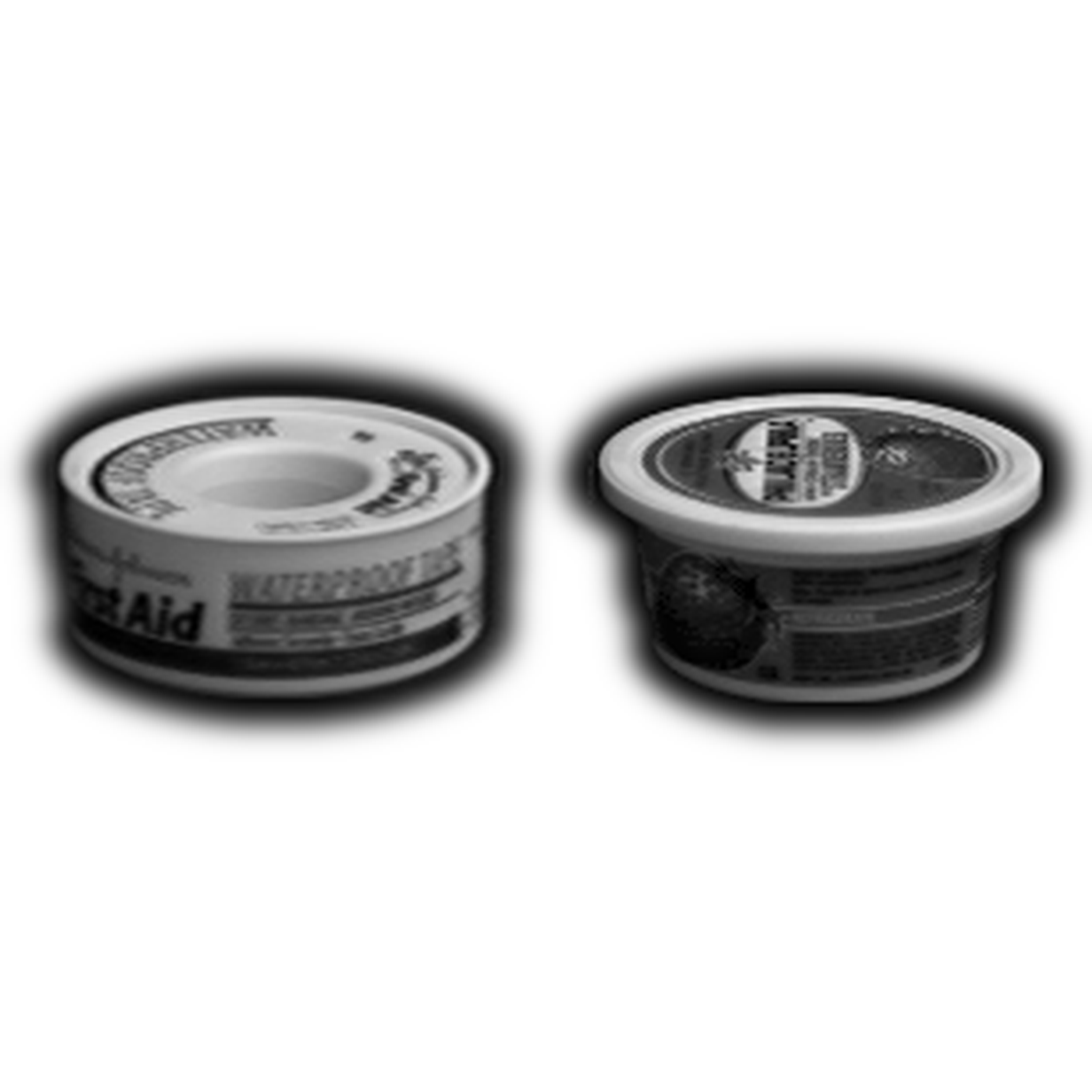}
        \vspace{0.5em}
        \caption{$\theta_A = 50^\circ, \, \theta_B = 105^\circ$}
        \label{fig:original_imageB}
    \end{subfigure}
    \begin{subfigure}[t]{0.24\textwidth}
        \centering
        \includegraphics[width=0.9\textwidth]{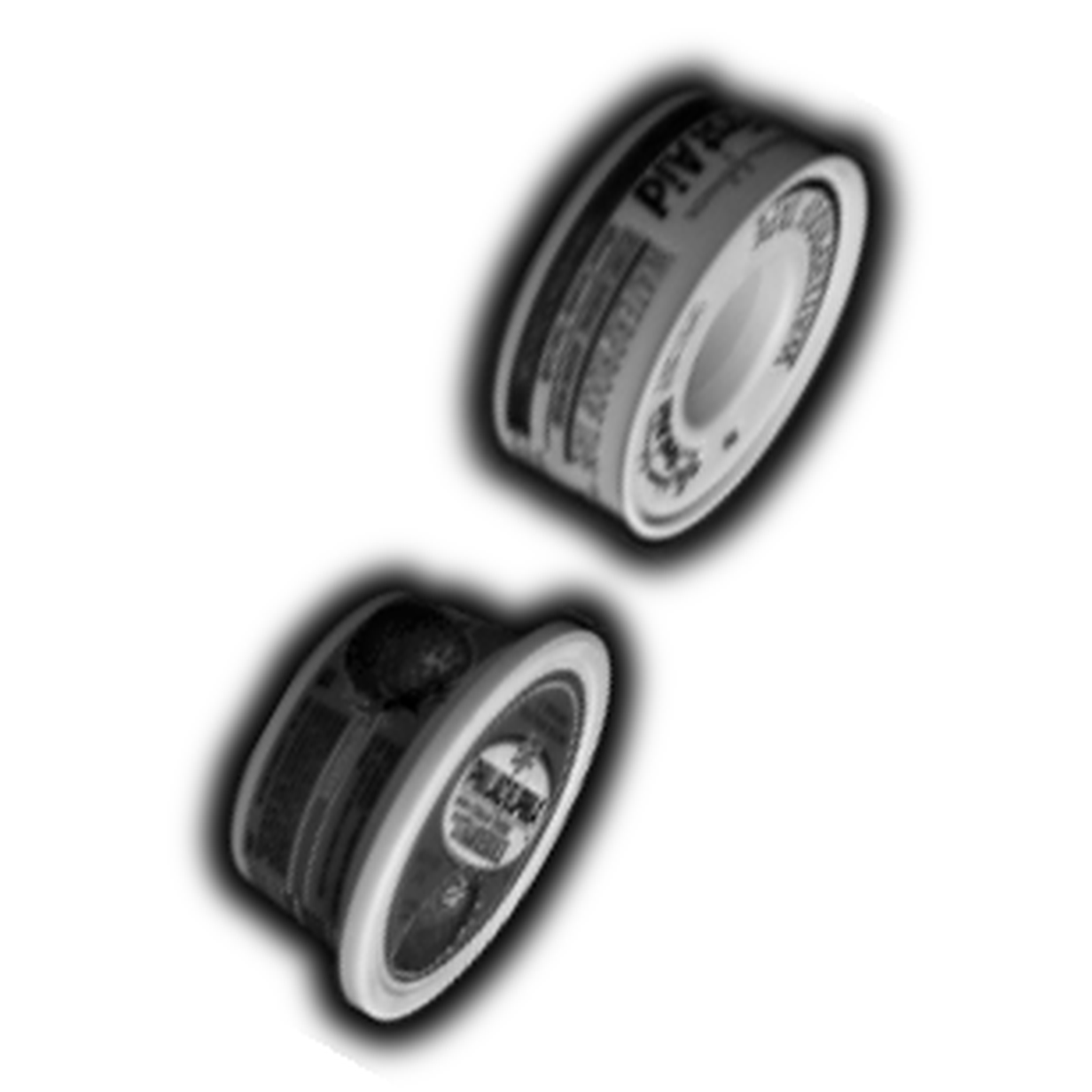}
        \vspace{0.5em}
        \caption{Outer rotation $240.8^\circ$}
        \label{fig:rotated_imageB}
    \end{subfigure}

    \vspace{0em}

    \caption{Two samples from the spinning objects experiment. In (a) and (c), COIL-100 objects 87 and 72 are composited at independently chosen poses; (b) and (d) show the same images after an in-plane rotation by a uniformly-drawn angle $\alpha$. 
   }
    \label{fig:spinning_toys}
\end{figure}

We performed four-dimensional spectral embeddings of the data using Algorithm~\ref{alg:spectral embedding} with each of the following kernels: Euclidean, minimum kernel, and integral kernel. In this case, the symmetry group is $G = \mathrm{SO}(2)$ and the bandwidth parameter was set to $\varepsilon = 200$ for each kernel. Figure~\ref{fig:spinning_toys} shows stereographic projections to $\RR^3$ of the resulting four-dimensional embeddings. The embedded data, colored according to each of the two intrinsic rotation angles, shows that both degrees of freedom are adequately captured and that the embedding recovers the correct torus-like intrinsic geometry.
This experiment shows the success of the method in a case where the intrinsic geometry is two-dimensional rather than one-dimensional.

\begin{figure}[h!]
\centering
\setlength{\tabcolsep}{2pt}
\begin{tabular}{lcccc}
\toprule
& Colored by $\theta_A$ & & Colored by $\theta_B$ & \\
\midrule
Minimum
  & \makecell{\includegraphics[width=0.36\textwidth]{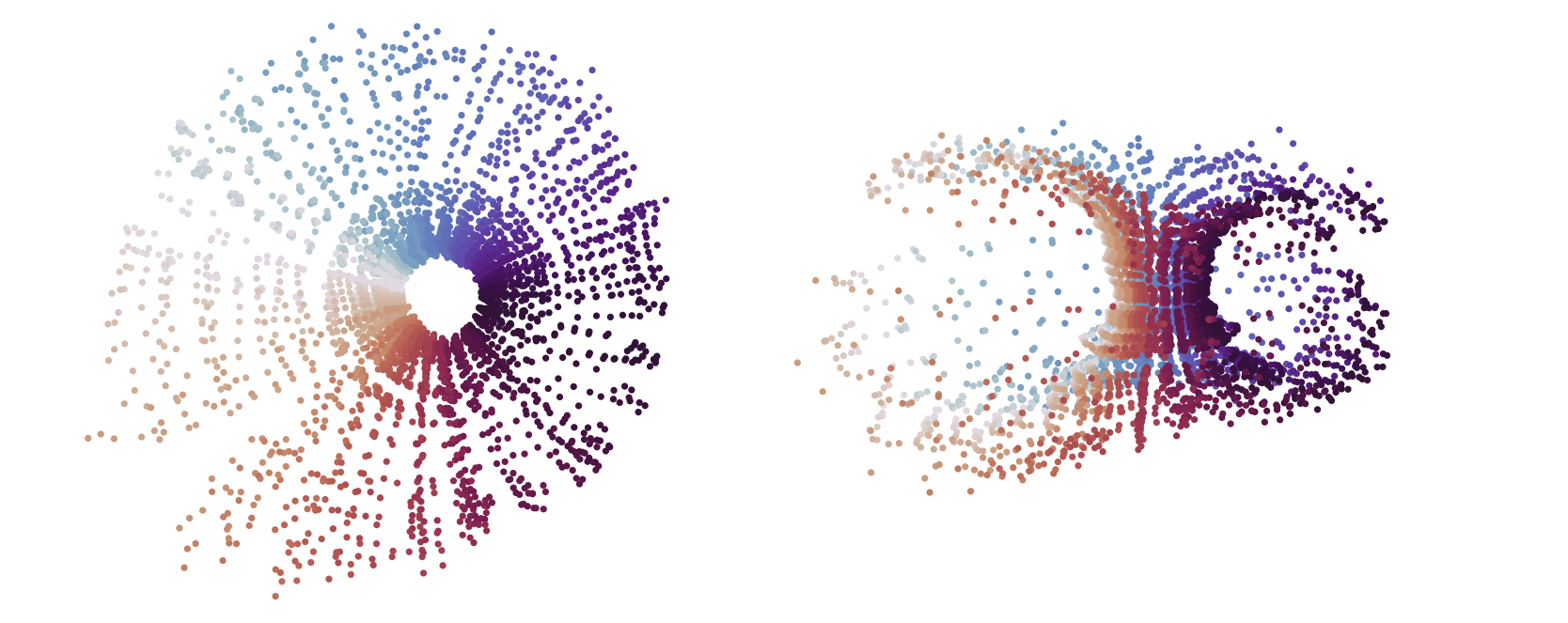}}
  & \multirow{3}{*}{\makecell{\includegraphics[height=0.12\textheight]{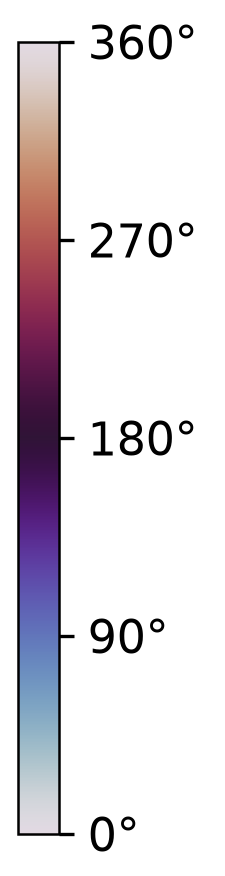}}}
  &\makecell{ \includegraphics[width=0.36\textwidth]{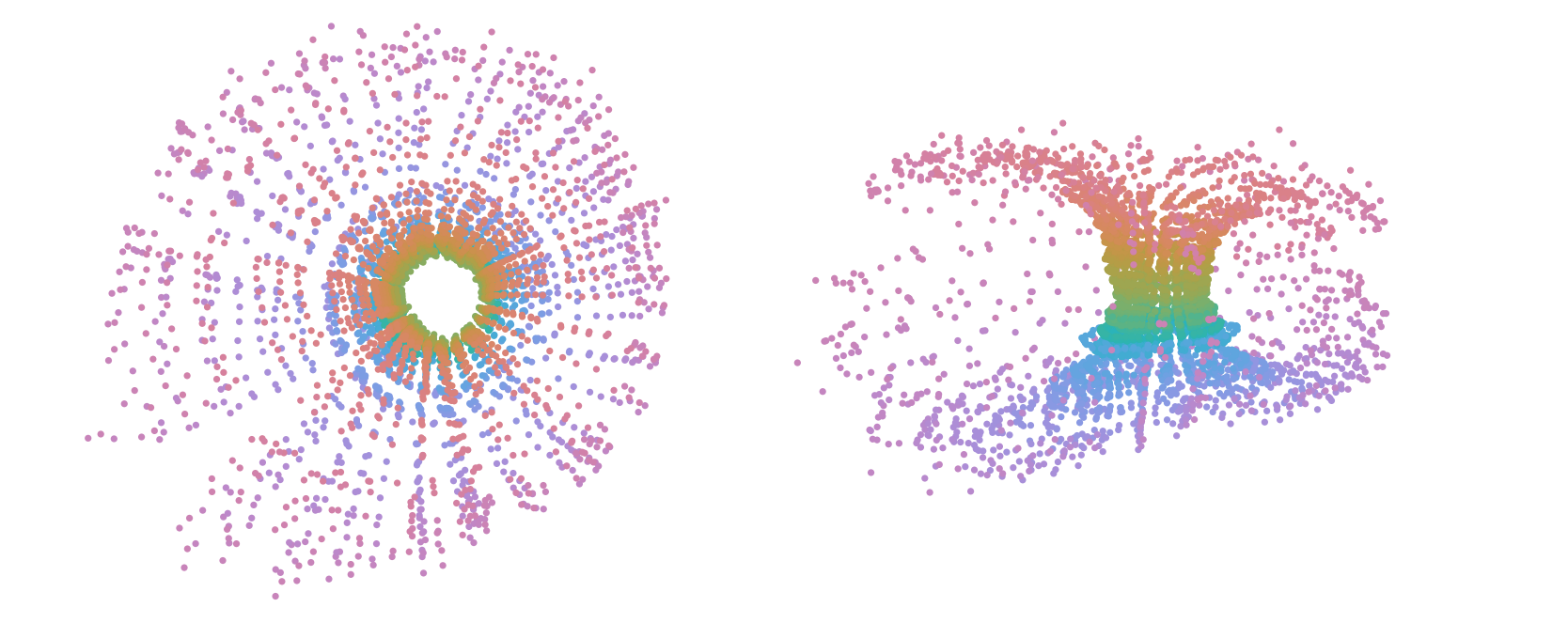}}
  & \multirow{3}{*}{\makecell{\includegraphics[height=0.12\textheight]{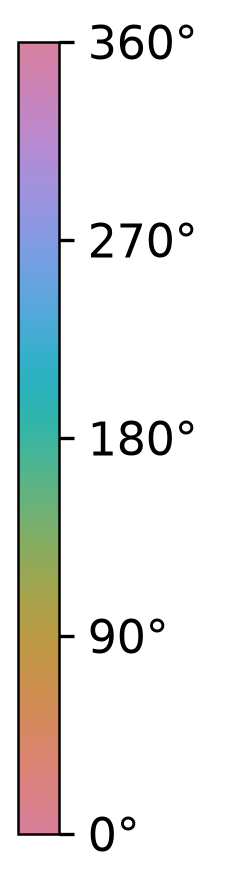}}} \\
Integral
  &\makecell{ \includegraphics[width=0.36\textwidth]{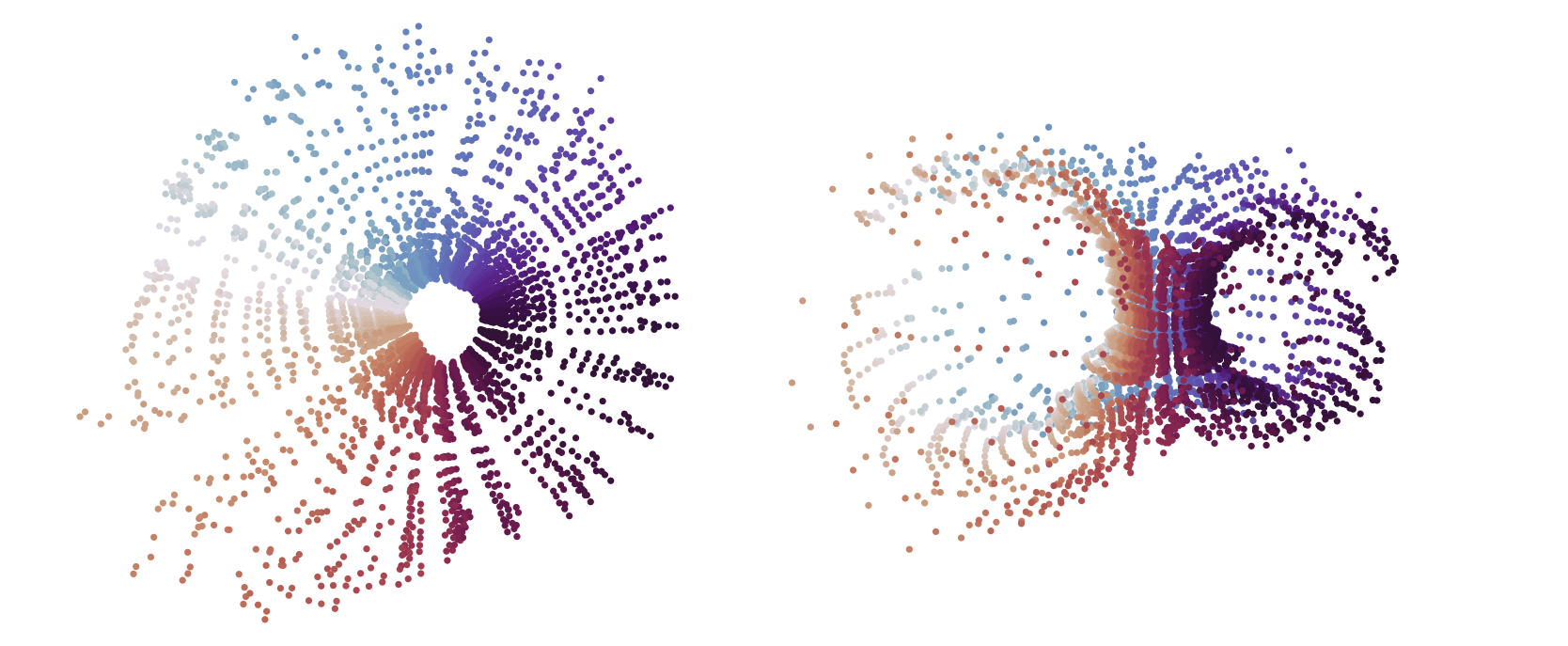}} & &
   \makecell{ \includegraphics[width=0.36\textwidth]{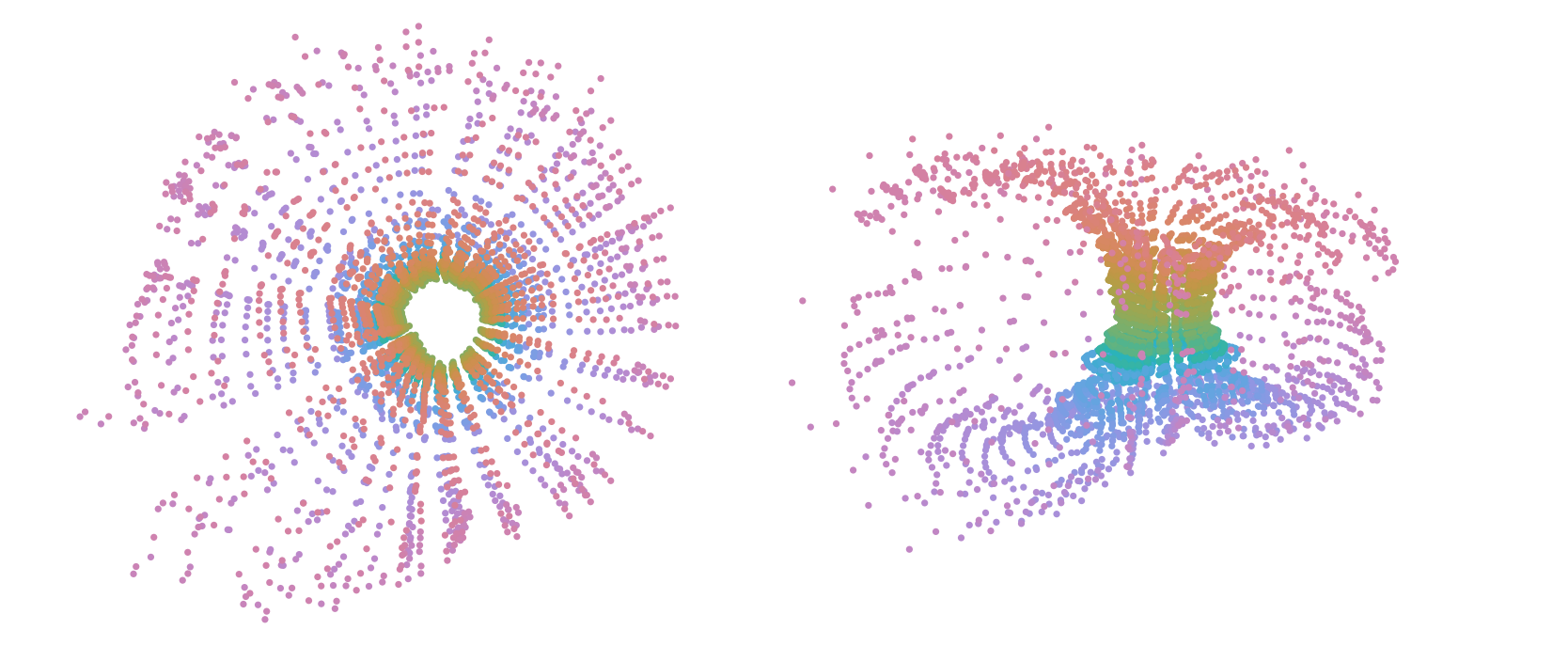} }& \\
Euclidean
  & \makecell{\includegraphics[width=0.36\textwidth]{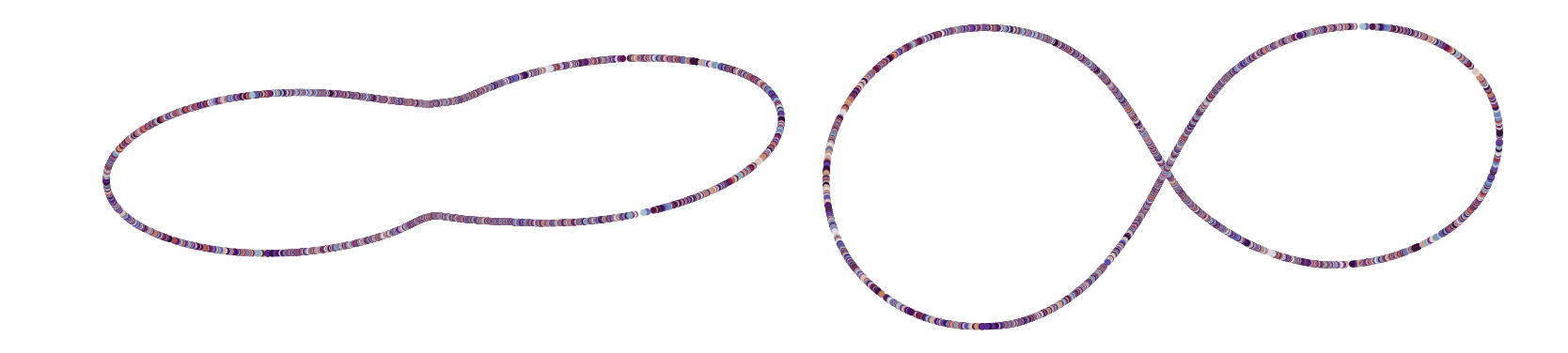}} & &
\makecell{\includegraphics[width=0.36\textwidth]{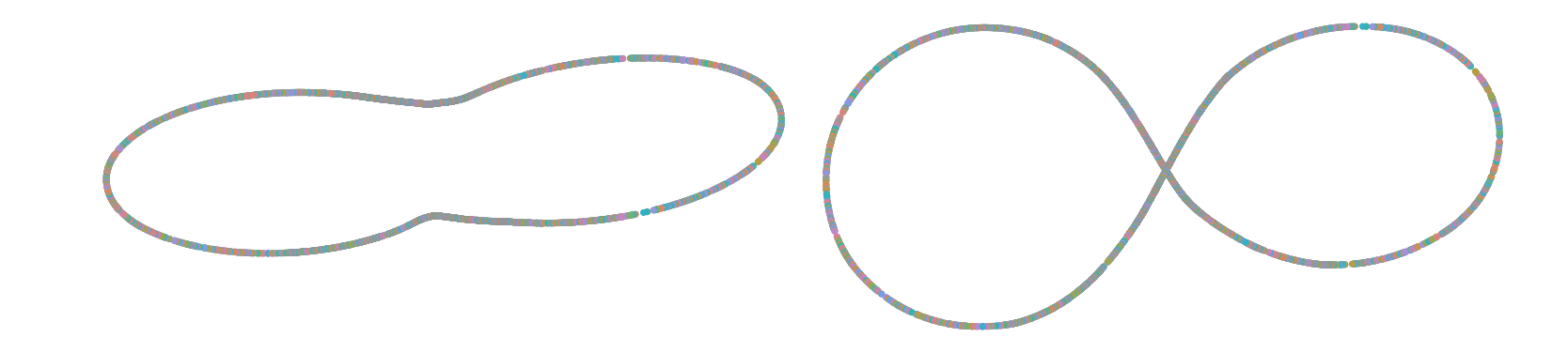}} & \\
\bottomrule
\end{tabular}
\caption{Stereographic projections to $\RR^3$ of the four-dimensional spectral embeddings obtained with the Euclidean kernel and with the two $\mathrm{SO}(2)$-invariant kernels. 
Each panel shows the same embedded points from a top and a side
viewpoint. Points are colored by the rotation angle $\theta_A$ of object
87, in the left column, and by the rotation angle $\theta_B$ of object 72,
in the right column. With the $\mathrm{SO}(2)$-invariant kernels, both intrinsic degrees of freedom are smoothly parameterized, and the torus-like geometry of $\mathbb{S}^1 \times \mathbb{S}^1$ is recovered. With the Euclidean kernel the embedding collapses onto a single closed curve, along which neither angle is resolved.
}
\end{figure}

\section{Conclusion}
We have proposed a framework for the spectral embedding of datasets with known symmetries 
by incorporating group invariance directly into the affinity kernels.
Our theoretical analysis focused on the setting of a Riemannian data manifold $\M$, with symmetries given by a compact Lie group $G$.  We proved that graph Laplacians constructed from three broadly applicable classes of $G$-invariant kernels converge pointwise to explicit second-order differential operators on the quotient manifold $\N = \M / G$. In all three cases, the differential operators include a first-order term related to the action of $G$ on $\M$.
The utility of the framework was demonstrated through experiments on datasets with $\mathrm{SO}(2)$ and $\mathrm{SO}(3)$ symmetries. Our experiments showed that  $G$-invariant affinity kernels recover the correct intrinsic geometry of the data, while the standard Euclidean kernel fails to do so even at large sample sizes.  The resulting symmetry-aware embeddings faithfully parameterize the degrees of freedom of the quotient manifold.
A key feature of our results is the improved convergence rate: the variance error drops from $n^{-1/2}\varepsilon^{-1/2-d/4}$ for the standard graph Laplacian to $n^{-1/2}\varepsilon^{-1/2-(d - p)/4}$ for the graph Laplacian constructed from a group invariant kernel,  effectively reducing the dimension of the problem by $p = \dim(G)$.
This is similar to the known improved convergence rates when using symmetry augmentations in supervised learning \citep{chen2020group,tahmasebi2023exact}.

Several directions remain open for future work. 
First, while we establish pointwise convergence of the graph Laplacian, convergence of the eigenvalues and eigenvectors to those of the limiting operator remains to be proven.
Second, our analysis assumes a free action of a compact Lie group, so that $\N = \M/G$ is itself a smooth manifold.  
Extending the framework to non-free actions, would broaden its applicability to cases such as $\mathrm{SO}(3)$ acting on symmetric molecules.
For non-free actions the orbit space is an orbifold with singular strata of varying isotropy.  
Third, it would be useful to consider situations where the data manifold has only approximate symmetries \cite{tahmasebi2025achieving}.

\subsection*{Acknowledgments}

We thank Nicolas Boumal, Yu-Fang Hsieh, Ofir Karin, Yael Karshon, Roy Lederman, Ofir Lindenbaum, Haggai Maron, Keren Mor Waknin, Amitai Netser Zenik,  Eitan Rosen, Yoel Shkolnisky, Amit Singer, and Nir Sochen for helpful discussions.
J.K. is partially supported by the United States National Science Foundation (Grant Nos. 2309782, 2436499 and 2312746), the Department of Energy (Grant No. SC0025312), and the Sloan Foundation.
A.M. is partially supported by the United States-Israel Binational Science Foundation
(Grant No. 2022778) and by the Israel Science Foundation (Grant No. 1662/22).
D.T. is partially supported by the Knut and Alice Wallenberg Foundation through a joint WASP--DDLS grant.
J.A. is partially supported by the Swedish Research Council (Grant No. 2023-04143).

\subsection*{AI Use Disclosure}
In this work, we used generative AI tools to assist with checking proofs and with developing and documenting the accompanying code. All proofs in this document were originally conceived of and written by the authors; AI tools helped verify derivations and identify errors in intermediate formulas. The numerical experiments were implemented by the authors with the help of AI coding assistants, which were used to draft parts of the original code, to identify errors in it and implement the corrections, to organize the repository (merging scripts by function and removing unused files), and to write the accompanying documentation. We have reviewed all AI-assisted work: every line has been examined and revised by the authors, and we take full responsibility for the final content of this work.

\bibliographystyle{elsarticle-num-names}
\bibliography{main}

\appendix
\renewcommand{\thesection}{\Alph{section}}
\renewcommand{\thetheorem}{\Alph{section}.\arabic{theorem}}

\section{Riemannian Quotient Manifold}\label{app:quotient_manifolds}
This appendix recalls some basic notation and results related to group actions and quotient manifolds.
For more on these topics, refer to the textbooks by \citet{lee2003introduction,lee2018Riemannian}.

Let $G$ be a compact Lie group with identity element denoted by $e$. A  \emph{left action} of $G$ on the Riemannian manifold $\M$  is a map $\theta: G \times \M \to \M$, denoted as $\alpha \cdot x = \theta(\alpha,x)$, such that $\alpha \cdot (\beta \cdot x) = (\alpha \beta) \cdot x$ and $e \cdot x = x$ for all $\alpha, \beta \in G$ and $x \in \M$. The group action of $G$ on $\M$ is said to be \emph{free} if
\begin{align}
    \exists x \in \M: \theta(\alpha,x) = x \quad \Longrightarrow \quad \alpha = e. 
\end{align}
i.e., non-identity actions have no fixed points. The action is called \emph{smooth} if $\theta(\alpha,x)$ is a smooth map with respect to both $\alpha$ and $\x$.
In that case, it induces a mapping of the tangent bundle $d\theta_{\alpha,x}: T_x \M \to T_{\theta(\alpha,x)} \M$ known as the \emph{differential} of $\theta_\alpha$ at $x$.
We say $G$ \emph{acts by isometries} if for every $\alpha \in G$ the diffeomorphism given by $\theta_\alpha$ preserves the Riemannian metric,
\begin{align}
    \langle u, v \rangle_x
    =
    \langle d\theta_{\alpha,x}(u), d\theta_{\alpha,x}(v) \rangle_{\theta(\alpha,x)}
   \!  \quad
    \text{for all } x \in \M \text{ and all } u, v \in T_x.
\end{align}

The \emph{orbit} of a point $x \in \M$ under the action of $G$ is the set
\begin{equation}
    [x] = G \cdot x = \{ \alpha \cdot x \mid \alpha \in G \} \subseteq \M.
\end{equation}
The orbits are the equivalence classes of the $G$-action, i.e. $x \sim y$ if and only if there exists $ \alpha \in G$ such that $y = \alpha \cdot x$.
As a set, the quotient $\M/G$ is the set of orbits of $G$, also known as the \emph{orbit space} of the action. The map $\pi: \M \to \M / G$ defined by $\pi(x) = [x]$ is called the \emph{quotient map}. 
We say that $\pi$ is a Riemannian submersion if it is a smooth map, its differential $d\pi$ is surjective at each point, and  
\begin{equation}
    g_\M(u,v) = g_\N(d\pi_x(u), d\pi_x(v))
\end{equation}
for all $u, v$ in the orthogonal complement to $\ker(d\pi_x)$. The following theorem gives sufficient conditions for $\pi$ to be a Riemannian submersion.
\begin{theorem}[Riemannian Quotient Manifold]\label{thm: Riemannian quotient manifold}
    If a compact Lie group $G$ acts smoothly,  freely and isometrically on $\M$, then the orbit space $\N \colonequals \M/G$ is a smooth manifold of dimension $\dim(\M)- \dim(G)$ with a unique Riemannian metric $g_\N$ such that the quotient map $\pi \colon \M \to \N$ is a Riemannian submersion.
\end{theorem}
A proof can be found in \citet[Cor 2.29]{lee2018Riemannian}, noting that the action of a compact group is always proper.
\noindent We call $(\N, g_\N)$ the \emph{Riemannian quotient manifold} coming from the action of $G$ on $\M$.

\section{Proof of Lemma~\ref{lem: Fubini formula for integration over M}}\label{app: proof of Fubini formula}
By Theorem~\ref{thm: Riemannian quotient manifold}, the quotient map $\pi \colon \M \to \N$ is a Riemannian submersion, and its fibers are given by $\pi^{-1}([x]) = G \cdot x$. Applying~\cite[Thm.\ 5.6]{Sakai1996} then gives
\begin{align}\label{eq: inner integral in terms of dV_Gx}
    \int_\M f(x)\,dV_\M(x) 
    = \int_{\N} \left(\int_{G\cdot x} f(\alpha \cdot x) dV_{G\cdot x}(\alpha \cdot x)\right) dV_\N([x]),
\end{align}
where $dV_{G\cdot x}$ is the Riemannian measure on $G \cdot x$ induced by the ambient metric on $\M$, which is related to the Haar measure $d\eta$ on $G$ by $dV_{G\cdot x}(\alpha \cdot x) = \delta(x)d\eta(\alpha)$ (Definition~\ref{def: density function}).
Thus Eq.~\eqref{eq: inner integral in terms of dV_Gx} becomes
\begin{align}\label{eq: inner integral in terms of d eta}
    \int_\M f(x)\,dV_\M(x) 
    = \int_{\N} \left(\int_{G} f(\alpha \cdot x) \delta(x) d\eta(\alpha)\right) dV_\N([x]). 
\end{align}
Since $\delta(x) = \bar\delta([x])$ depends only on the orbit $G\cdot x$, this 
gives Eq.~\eqref{eq: Fubini integral on M}.
To conclude the proof, let us further assume that $f \colon \M \to \RR$ is a $G$-invariant function. Let $\bar{f}$ denote the induced function on $\N$. By substituting  $f(\alpha \cdot x) = \overline{f}([x])$ into Eq. \eqref{eq: Fubini integral on M},
and since the Haar measure $d\eta$ is chosen so that $\eta(G) = 1$, we obtain Eq.~\eqref{eq: Fubini integral on M for G-invariant function} as desired.

\section{Proof of Lemma \ref{lem: integration wrt minimum kernel}}\label{app:integration wrt minimum kernel}
We derive an asymptotic expansion of the integral
\begin{equation}\label{eq: integral minimum kernel against smooth h}
    I([x]) = \frac{1}{(\pi\varepsilon)^{q/2}}\int_\N \overline{K}_{\mathrm{min}}([x],[y])h([y])dV_\N([y]),
\end{equation}
with
\begin{equation}
     \overline{K}_{\mathrm{min}}([x],[y])
     =
     \exp
     \left(
        -\min_{\alpha \in G} \| x - \alpha \cdot y \|^2/\varepsilon
    \right),
\end{equation}
The exponential map then lets us rewrite the approximation of $I([x])$ as an integral over $q$-dimensional Euclidean space. Finally, standard properties of  the Gaussian distribution in Euclidean space are used to conclude the proof.

\subsection*{Step 1: Comparison of Riemannian and Euclidean Distances on $\M$}
In this step, we prove three inequalities that relate the Riemannian distance $d_\M(x,y)$ and the Euclidean distance $\|x-y\|$ for any two points $x,y \in \M$. These inequalities imply that the Riemannian and Euclidean distances are uniformly equivalent metrics on $\M$.
\begin{lemma}\label{lem: inequality geodesic vs chordal distance}
Let $(\M, g_\M)$ be a compact and connected Riemannian submanifold of $\RR^D$. Then there exist positive constants $C_\M$ and $K_\M$ such that 
\begin{align}
    0 \leq d^2_\M (x,y) - \|x - y\|^2 \leq C_\M d^4_\M(x,y) \leq K_\M \|x - y\|^4
\end{align}
for all $x,y \in \M$.
\end{lemma}
\begin{proof}
The first inequality follows directly from the definition of the Riemannian and Euclidean distances on $\M$. 
To prove the second inequality, note that, since $\M$ is compact and connected, the global injectivity radius $r_{\M}$ is a positive constant  such that for all  $x, y \in \M$ with $d_{\M}(x, y) < r_{\M}$ 
there is a unique minimizing geodesic $\gamma_{xy}^\M$ in $\M$ starting at $x$ and ending at $y$. 
Assume that $\gamma_{xy}^\M$ is parametrized by arc-length.  As  proved in \citep[Proposition 6]{Smolyanov2007}, the following limit holds uniformly with respect to $x, y \in \M$: 
\begin{align}
    \lim_{d_\M(x,y) \to 0} \frac{d_\M^2(x,y) - \|x-y\|^2 }{d_\M^4(x,y)} = \frac{1}{12} \left\| \ssf(\dot{\gamma}^{\M}_{xy}(0), \dot{\gamma}^{\M}_{xy}(0)) \right\|^2,
\end{align}
where $\ssf$ denotes the second fundamental form of $\M$. Let
\begin{align} C = \frac{1}{12} \sup_{v \in U\M} \left\| \ssf(v, v) \right\|^2 + 1,\end{align}
where $U\M$ denotes the unit tangent bundle of $\M$; the supremum is finite since $\ssf$ is continuous and $U\M$ is compact. By the uniform convergence of the limit above, there exists a constant $\eta \in (0, r_{\M})$ such that if $0 < d_\M(x,y) < \eta$ then
\begin{align}
        \frac{d_\M^2(x,y) - \|x-y\|^2}{d_\M^4(x,y)} \leq \frac{1}{12}\left\| \ssf(\dot{\gamma}^{\M}_{xy}(0), \dot{\gamma}^{\M}_{xy}(0)) \right\|^2 + 1 \leq C.
\end{align}
On the other hand, if $\eta \leq d_\M(x,y)$, then
\begin{align}
    d_\M(x,y)^2 - \|x-y\|^2 \leq d_\M(x,y)^2 \leq \frac{1}{\eta^2}d_\M(x,y)^4.
\end{align}
In any case, for all pairs $x,y \in \M$ we have 
\begin{align}\label{eq: bound1}
    d_\M(x,y)^2 - \|x-y\|^2 \leq C_\M d_\M(x,y)^4,
\end{align}
where $C_\M = \max\{1/\eta^2, C\}$. Finally, we use inequality~\eqref{eq: bound1} to prove that there exists a nonzero constant $K$ such that
\begin{align}\label{eq: bound2}
    d_\M(x,y) \leq K \|x - y\|
\end{align}
for all $x,y \in \M$. The last inequality then follows by taking $K_\M =  C_\M K^4$. 

Assume, for the sake of contradiction, that for all $N \in \mathbb{N}$ there exist $\x_N, \y_N \in \M$ such that
\begin{align}\label{eq: contradiction hypothesis}
    d_\M^2(\x_N, \y_N) > N \|\x_N-\y_N\|^2.
\end{align}
Due to the compactness of $\M$, and by passing to a convergent subsequence if necessary, we may assume that $(\x_N, \y_N) \to (x^\ast,y^\ast) \in \M \times \M$ as $N \to \infty$. Eq.~\eqref{eq: contradiction hypothesis} implies $\|\x_N - \y_N\| < d_\M(\x_N, \y_N)/ \sqrt{N} < D/\sqrt{N} \to 0$, where $D =\mathrm{diam}(\M) < \infty$ by compactness of $\M$, so it follows that $x^\ast = y^\ast$. Combining Eqs.~\eqref{eq: bound1} and~\eqref{eq: contradiction hypothesis} gives
\begin{align}
    d_\M^2(\x_N,\y_N)  > N\|\x_N-\y_N\|^2 \geq Nd_\M^2(\x_N,\y_N) - NC_\M d_\M^4(\x_N,\y_N).
\end{align}
For all $N > 1$, this implies
\begin{align}
    \frac{1}{d_\M(\x_N, \y_N)^2} < \frac{NC_\M}{N - 1}.
\end{align}
This gives a contradiction as $N \to \infty$, since $x^\ast = y^\ast$ implies $1/d_\M^2(\x_N, \y_N) \to \infty$ but $N C_\M /(N-1) \to C_\M$.
\end{proof}

\subsection*{Step 2: Approximation of the Minimum Kernel
}
We now use Lemma \ref{lem: inequality geodesic vs chordal distance} to show that $\min_{\alpha \in G}\|x-\alpha\cdot y\|^2$ approximates the Riemannian distance $d_\N([x], [y])$ up to order three.  We then use this approximation to obtain an asymptotic expansion of the minimum kernel as $\varepsilon \to 0$. 
Let us start by noting that Lemma \ref{lem: inequality geodesic vs chordal distance} gives
\begin{align}\label{eq: inequality 1}
    \|x - \alpha \cdot y\|^2 \leq d_\M^2(x, \alpha \cdot y)
\end{align}
and
\begin{align}\label{eq: inequality 2}
d_\M^2(x,\alpha \cdot y) \leq \|x - \alpha \cdot y\|^2 + K_\M  \|x- \alpha \cdot y\|^4 
\end{align}
for all $x, y \in \M$ and all $\alpha \in G$. On the other hand,
the Riemannian distances on $\M$ and $\N$ are related by the formula
\begin{equation}\label{eq: relation between d_M and d_N}
    d_\N([x],[y]) = \min_{\alpha \in G} d_\M(x, \alpha \cdot y)
\end{equation}
for all $[x], [y] \in\N$ (see \cite[Exercise 10.15]{boumal2023intromanifolds}). 
Although the minimizing group elements for the Euclidean and Riemannian distances need not coincide, the inequalities in Lemma~\ref{lem: inequality geodesic vs chordal distance} hold uniformly for every $\alpha\in G$. Consequently,
taking the minimum over $\alpha \in G$  in inequalities \eqref{eq: inequality 1} and \eqref{eq: inequality 2}, using relation~\eqref{eq: relation between d_M and d_N}, and combining the results yields
\begin{align}\label{eq: bound on min||x - alpha y||^2}
- d^2_\N([x], [y]) \leq - \min_{\alpha \in G} \|x - \alpha\cdot y\|^2 
\leq - d^2_\N([x], [y])  + K_\M d^4_\N([x], [y])
\end{align}
for all $[x], [y] \in \N$. 
It follows that, when $ d_\N([x],[y])<\varepsilon^{\gamma}$ with $1/4 < \gamma$, the minimum kernel satisfies
\begin{equation}
\begin{aligned}
    \overline{K}_{\mathrm{min}}([x],[y])
    =
    \exp\!\left(
        - \min_{\alpha \in G}\|x - \alpha \cdot y\|^2 / \varepsilon
    \right)
    &=
    \exp\!\left(
        -d^2_\N([x],[y]) / \varepsilon
    \right)
    \left(
        1
        +
        O\!\left(
            d_\N^4([x],[y])/\varepsilon
        \right)
    \right) 
\end{aligned}
\label{eq: approximation of the minimum kernel}
\end{equation}
for all sufficiently small $\varepsilon$, where the implied constant is $K'_\M = \exp(K_\M) -1$.

\subsection*{Step 3: Reduction of the Domain of Integration}
 As proved in Lemma~\ref{lem: inequality geodesic vs chordal distance}, there exists a positive constant $K > 0$
such that $ d_\M(x, \alpha \cdot y) \leq K \|x - \alpha\cdot y\|$  for all $x,y \in \M, \alpha \in G$. Since  this inequality holds uniformly for every $\alpha \in G$, taking the minimum over $\alpha \in G$ gives 
\begin{equation}\label{eq: lower bound on min |x - alpha y|}
    d_\N([x],[y]) = \min_{\alpha \in G} d_\M(x, \alpha \cdot y) \leq K \min_{\alpha \in G} \|x - \alpha\cdot y\|.
\end{equation}
Let $B$ denote the geodesic ball in $\N$ around $[x]$ of radius $\varepsilon^\gamma$ with $1/4 < \gamma < 1/2$. Then Eq.~\eqref{eq: lower bound on min |x - alpha y|} implies  that for any $[y] \in \N \setminus B$ we have
\begin{equation}
     \min_{\alpha \in G} \|x - \alpha\cdot y\|^2 \geq \frac{1}{K^2} d^2_\N([x], [y])   \geq \frac{1}{K^2} \varepsilon^{2\gamma}.
\end{equation}
It follows that restricting the domain of integration of $I([x])$ in Eq.~\eqref{eq: integral minimum kernel against smooth h} to  $B$ generates an error of order
\begin{equation}
\begin{aligned}
    \left\lvert
        \frac{1}{(\pi\varepsilon)^{q/2}}
        \int_{\N \setminus B}
        \exp\!\left(
            - \min_{\alpha \in G}\|x - \alpha \cdot y\|^2/\varepsilon
        \right)
        h([y])\, dV_\N ([y])
    \right\rvert 
    \leq
    \frac{\vol(\N \setminus B)}{(\pi\varepsilon)^{q/2}}   \,\|h\|_\infty\, 
    \exp\left(
        -\varepsilon^{2\gamma-1}/K^2
    \right),
\end{aligned}
\label{eq:kernel-bound}
\end{equation}
where $\|h\|_\infty <\infty$ since $h$ is a continuous function on a compact space. 
Given that $2\gamma -1 < 0$, this error is exponentially small as $\varepsilon$ approaches $0$. Therefore, for any $\eta> 0$ and for all sufficiently small $\varepsilon$, we have
\begin{align}\label{eq:  I([x]) restricted to U}
\begin{split}
    I([x]) 
    &=
    \frac{1}{(\pi\varepsilon)^{q /2}}
    \int_B
    \exp\!
    \left(
        - \min_{\alpha \in G}\|x - \alpha \cdot y\|^2/\varepsilon
    \right)
    h([y])dV_\N([y]) + o(\varepsilon^\eta). 
\end{split}
\end{align} 
Finally, since $1/4 < \gamma$,  the approximation of the minimum kernel in Eq.~\eqref{eq: approximation of the minimum kernel} applies to all $[y] \in B$. Substituting Eq.~\eqref{eq: approximation of the minimum kernel} into Eq.~\eqref{eq:  I([x]) restricted to U} thus gives, for any $\eta> 0$ and all sufficiently small $\varepsilon$, 
\begin{align}\label{eq:  I([x]) restricted to U and approx of min kernel}
    I([x])
    =
    \frac{1}{(\pi\varepsilon)^{q /2}}
    \int_B
    \exp\!
    \left(
        -d^2_\N([x],[y])/\varepsilon
    \right)
    \left(
        1
        +
        O\!\left(
            d_\N^4([x],[y])/\varepsilon
        \right)
    \right)
    h([y])dV_\N([y])
    +
    o(\varepsilon^\eta).
\end{align}

\subsection*{Step 4: Change to Geodesic Normal Coordinates}
In this step, we use the exponential map and the corresponding geodesic normal coordinates to express $I([x])$ as an integral over $q$-dimensional Euclidean space. 
The exponential map at $[x] \in \N$, denoted $\exp_{[x]}$, 
provides a canonical smooth map from the tangent space at $[x]$ into the manifold $\N$. It maps straight lines through the origin in $T_{[x]}\N$ to geodesics in $\N$ through the point $[x]$. 
Moreover, $\exp_{[x]}$ is a local diffeomorphism, meaning that there exist open neighborhoods $U$ of the origin in $T_{[x]}\N$ and $V$ of $[x]$ in $\N$ such that $\exp_{[x]} \colon U \to V$ is a diffeomorphism.
By fixing an orthonormal basis on $T_{[x]}\N$ and letting $u = (u_1, \dots, u_q)^T$ denote coordinates on $U$ with respect to this basis, we obtain the  
\emph{geodesic normal coordinates} for $\N$ centered at $[x]$. In these coordinates, for each $[y] \in V$, there exists $u \in \RR^q$ such that 
\begin{equation}\label{eq: distance in geodesic coordinates}
    d^2_\N([x], [y]) = \|u\|^2.
\end{equation}
In addition, given any function $f \colon \N \to \RR$, we let $\tilde{f}(u)$ denote the local coordinate expression for $f([y])$. 

Note that there exists some fixed $\varepsilon_0 >0$ such that for all $\varepsilon < \varepsilon_0$ the open ball $\tilde B$ of radius $\varepsilon^{\gamma}$ with $1/4 < \gamma < 1/2$ around the origin in $T_{[x]}\N$ is contained in $U$. In this case, $\exp_{[x]}$ maps $\tilde B$ diffeomorphically to $B$, the geodesic ball around $[x]$ of radius $\varepsilon^\gamma$. 
From Eqs.~\eqref{eq:  I([x]) restricted to U and approx of min kernel} and~\eqref{eq: distance in geodesic coordinates},  we see that $I([x])$ is given in geodesic normal coordinates by
\begin{align}\label{eq: int exponential coordinates}
    I([x])
    =
    \frac{1}{(\pi\varepsilon)^{q /2}}
    \int_{\tilde B}
    \exp
    \left(
        -\|u\|^2/\varepsilon
    \right)
    \left(
        1
        +
        O
        \left(
            \|u\|^4/\varepsilon
        \right)
    \right)
    \tilde{h}(u)
    \sqrt{\det(g_{ij})}du
    +
    o(\varepsilon^\eta),
\end{align}
where $\sqrt{\det(g_{ij})}du$ is the volume form for $\N$. Following~\cite[Corollary 3]{Smolyanov2007},  we have the Taylor expansion
\begin{align} \label{eq: expansion of volume form in normal coordinates}
\sqrt{\det(g_{ij})} = 1 - \frac{1}{6} u^T \tilde{R}(0) u + O(\|u\|^3), 
\end{align}
where $R$ denotes the Ricci curvature tensor of $\N$. Since $\N$ is a compact manifold, the elements of $R$ are bounded, so that we can further approximate
\begin{align} \label{eq: further expansion of volume form in normal coordinates}
\sqrt{\det(g_{ij})} = 1 + O(\|u\|^2).
\end{align}
We will use both expressions for $\sqrt{\det(g_{ij})}$ to find the desired approximation for $I([x])$. 

\subsection*{Step 5: Analysis in Euclidean Space}
Next, consider the Taylor expansion of $\tilde{h}(u)$ around $0$, which is given by
\begin{equation}\label{eq: Taylor expansion of h in normal coordinates}
    \tilde{h}(u) = \tilde{h}(0) + \sum_{i=1}^q u_i \frac{\partial \tilde{h}(0)}{\partial u_i} + \frac{1}{2}\sum_{i=1}^q \sum_{j=1}^q u_i u_j \frac{\partial^2 \tilde{h}(0)}{\partial u_i \partial u_j} +  O(\|u\|^3).
\end{equation}
Putting Eqs.~\eqref{eq: int exponential coordinates}-\eqref{eq: Taylor expansion of h in normal coordinates} together, integral $I([x])$ becomes
\begin{equation}
\begin{aligned}
I([x]) &= \tilde{h}(0)\,\frac{1}{(\pi\varepsilon)^{q/2}}
   \int_{\tilde B} \exp\!\left(-\frac{\|u\|^2}{\varepsilon}\right) \left( 1 + O\left(\frac{\|u\|^4}{\varepsilon} \right)\right) 
   \left(1 - \frac{1}{6}u^T \tilde{R}(0)u + O(\|u\|^3)\right)du \\[1ex]   
 &\quad+ \sum_{i=1}^q \frac{\partial \tilde{h}(0)}{\partial u_i} \frac{1}{(\pi\varepsilon)^{q/2}}
   \int_{\tilde B} u_i \exp\!\left(-\frac{\|u\|^2}{\varepsilon}\right)\left( 1 + O\left(\frac{\|u\|^4}{\varepsilon} \right)\right)  
   \left(1 + O(\|u\|^2)\right) du \\[1ex]
 &\quad+ \frac{1}{2}\sum_{i=1}^q \sum_{j=1}^q \frac{\partial^2 \tilde{h}(0)}{\partial u_i \partial u_j} \frac{1}{(\pi\varepsilon)^{q/2}}
   \int_{\tilde B} u_i u_j \exp\!\left(-\frac{\|u\|^2}{\varepsilon}\right) \left( 1 + O\left(\frac{\|u\|^4}{\varepsilon} \right)\right) 
   \left(1 + O(\|u\|^2) \right)du \\[1ex]
   &\quad+ \frac{1}{(\pi\varepsilon)^{q/2}}
   \int_{\tilde B} O({\|u\|^3})\exp\!\left(-\frac{\|u\|^2}{\varepsilon}\right) \left( 1 + O\left(\frac{\|u\|^4}{\varepsilon} \right)\right)  \left(1 + O(\|u\|^2)\right) du + o(\varepsilon^\eta).
\end{aligned}
\label{eq:I-expansion}
\end{equation}
By the properties of a zero-mean Gaussian distribution with diagonal covariance matrix, all terms of the form $\int_{\tilde B} u_i \exp(-\|u\|^2/\varepsilon)du$,  $\int_{\tilde B} u_i u_j u_k \exp(-\|u\|^2/\varepsilon)du$ and $\int_{\tilde B} u_iu_j \exp(-\|u\|^2/\varepsilon)du$ with $i \neq j$ vanish. Therefore, $I([x])$ simplifies to
\begin{equation}\label{eq: I([x]) after symmetry argument}
\begin{aligned}
I([x]) &= \tilde{h}(0)\,\frac{1}{(\pi\varepsilon)^{q/2}}
   \int_{\tilde B} \exp\!\left(-\frac{\|u\|^2}{\varepsilon}\right) \left( 1 + O\left(\frac{\|u\|^4}{\varepsilon} \right)\right) 
   \left(1 - \frac{1}{6} \sum_{i=1}^q \tilde{R}_{ii}(0) u_i^2
    + O(\|u\|^3)\right)du \\[1ex]
&\quad+ \frac{1}{2}\sum_{i=1}^q \frac{\partial^2 \tilde{h}(0)}{\partial u_i^2} \frac{1}{(\pi\varepsilon)^{q/2}}
   \int_{\tilde B} u_i^2  \exp\!\left(-\frac{\|u\|^2}{\varepsilon}\right)
  \left( 1 + O\left(\frac{\|u\|^4}{\varepsilon} \right)\right) 
   \left(1 + O(\|u\|^2) \right)du \\[1ex]
    &\quad+ \frac{1}{(\pi\varepsilon)^{q/2}}
   \int_{\tilde B} O({\|u\|^3})\exp\!\left(-\frac{\|u\|^2}{\varepsilon}\right) \left( 1 + O\left(\frac{\|u\|^4}{\varepsilon} \right)\right)  \left(1 + O(\|u\|^2)\right) du + o(\varepsilon^\eta).
\end{aligned}
\end{equation}
Due to the exponential decay of $\exp(-\|u\|^2/\varepsilon)$, the domains of integration 
can be extended to $\RR^q$ due to bounds similar to \eqref{eq:kernel-bound}. This allows us to use the following properties of the $q$-dimensional Gaussian distribution:
\begin{align}\label{eq: Gaussian integral 1}
    \frac{1}{(\pi\varepsilon)^{q /2}} \int_{\RR^{q}} \exp\left( -\frac{\|u\|^2 }{\varepsilon}\right)  du = 1, \quad \frac{1}{(\pi\varepsilon)^{q/2}} \int_{\mathbb{R}^q} \exp\!\left(-\frac{\|u\|^2}{\varepsilon}\right) \|u\|^m \, du = \varepsilon^{m/2}\, \frac{\Gamma\!\left(\frac{q+m}{2}\right)}{\Gamma\!\left(\frac{q}{2}\right)},
\end{align}
and, for each $u_i$,
\begin{align}\label{eq: Gaussian integral 2}
    \frac{1}{(\pi\varepsilon)^{q /2}} \int_{\RR^{q}} \exp\left( -\frac{\|u\|^2 }{\varepsilon}\right)  u_i^2 du = \frac{\varepsilon}{2}, 
    \quad \frac{1}{(\pi\varepsilon)^{q/2}} \int_{\mathbb{R}^q} u_i^2 \exp\!\left(-\frac{\|u\|^2}{\varepsilon}\right) \,\|u\|^m \, du = \frac{\varepsilon^{m/2 +1}}{q}\, \frac{\Gamma\!\left(\frac{q+m+2}{2}\right)}{\Gamma\!\left(\frac{q}{2}\right)}.
\end{align}
From Eqs.~\eqref{eq: Gaussian integral 1} and~\eqref{eq: Gaussian integral 2},
we approximate Eq.~\eqref{eq: I([x]) after symmetry argument} (choosing $\eta = 3/2$) up to order $\varepsilon^{3/2}$ 
 as follows:
\begin{align}\label{eq: final I([x]) in normal coordinates}
    I([x]) = \tilde{h}(0) + \frac{\varepsilon}{4}\left( \left( -\frac{1}{3}\sum_{i=1}^q \tilde{R}_{ii}(0) + 
    O(1) \right)\tilde{h}(0) 
    + \sum_{i=1}^q \frac{\partial^2 \tilde{h}(0)}{\partial u_i^2}\right)+ O(\varepsilon^{3/2}).
\end{align}

We briefly explain why the error term of $O(\epsilon^{3/2})$ in Eq.~\eqref{eq: final I([x]) in normal coordinates} can be improved,
following \citet{Singer2006}.
First note that the minimum kernel is smooth near the diagonal.  
Indeed, since $G$ acts freely the orbit $G\cdot x$ is a compact embedded submanifold of $\mathbb{R}^D$. 
By the tubular neighborhood theorem \cite{lee2003introduction}, $G \cdot x$ has an open neighborhood in $\mathbb{R}^D$ on which every point has a unique nearest point in $G \cdot x$ where the nearest-point projection to $G \cdot x$ is smooth \cite{foote1984regularity};
hence $y \mapsto \min_{\alpha \in G} \|x - \alpha \cdot y\|^2$ is smooth there. 
From
$$
 \overline{K}_{\mathrm{min}}([x],[y])
     =
     \exp
     \left(
        -\min_{\alpha \in G} \| x - \alpha \cdot y \|^2/\varepsilon
    \right),
$$
and compactness, $\overline{K}_{\min}$ is smooth in an open neighborhood of the diagonal of $\mathcal N\times\mathcal N$.  
This implies that the integrand in Eq.~\eqref{eq: int exponential coordinates} is smooth in geodesic normal coordinates, and proceeding as in~\citep[Section 2]{Singer2006} only integer powers of $\varepsilon$ survive. 
Therefore, the error term in Eq.~\eqref{eq: final I([x]) in normal coordinates} can be improved to $O(\varepsilon^2)$.

Finally, as shown by \citet{Rosenberg_1997}, we have the relation
\begin{align}
\sum_{i=1}^q \frac{\partial^2 \tilde{h}(0)}{\partial u_i^2} = - \Delta_\N h([x]),
\end{align}
where $\Delta_\N$ is the Laplace--Beltrami operator on $\N$. 
We conclude that
\begin{align}\label{eq: final I([x]) in terms of [x]}
    I([x]) = h([x]) + \frac{\varepsilon}{4}\left(E([x]) h([x]) - \Delta_\N h([x]) \right) + O(\varepsilon^{2}),
\end{align}
where $E([x]) = -\frac{1}{3} \sum_{i=1}^q R_{ii}([x]) + O(1)$ is some function that depends on the curvature of $\N$ at $[x]$ and on the embedding of $\M$ into $\RR^D$.

\newpage
\section{Table of Notation} \label{sec:notation}

\small
\begin{longtable}{@{}p{0.27\textwidth}p{0.67\textwidth}@{}}
\toprule
Notation & Meaning \\
\midrule
\endfirsthead
\toprule
Notation & Meaning \\
\midrule
\endhead
\midrule
\endfoot
\bottomrule
\endlastfoot
\multicolumn{2}{@{}l}{\textit{Spectral embedding}}\\[0.2em]
$\X=\{\x_1,\dots,\x_n\}$ & Sampled dataset, with $\x_i \in \M \subseteq \RR^D$. \\
$K,K_G, K_\mathrm{min}, K_\mathrm{int}, K_\mathrm{IF}$ & Base affinity kernel (e.g.\ $\exp(-\|x-y\|^2/\varepsilon)$), generic $G$-invariant kernel, minimum kernel, integral kernel, and invariant-features kernel (Section~\ref{sec:G-invariant kernels}).\\
$\varepsilon$ & Kernel bandwidth parameter. \\
$W, D$ & Weight and degree matrices used to define the graph Laplacian. \\
$L_{RW}$ & Random-walk normalized graph Laplacian, $L_{RW}=I-D^{-1}W$. \\
$\lambda_i,\varphi_i$ & Eigenvalues and eigenvectors of the graph Laplacian (used as embedding coordinates). \\
$\y_i$ & Embedded coordinate vector of $\x_i$ in spectral embedding. \\
\midrule
\multicolumn{2}{@{}l}{\textit{Spaces, geometry, and group actions}}\\[0.2em]
$\M \subseteq \RR^D$ & Data manifold. A compact Riemannian submanifold without boundary. \\
$d = \dim(\M)$ & Intrinsic dimension of the data manifold. \\
$G$ & Compact Lie group acting on $\M$. \\
$p = \dim(G)$ & Dimension of the Lie group. \\
$\alpha,\beta \in G$ & Generic group elements. \\
$\alpha \cdot x$ & Action of $\alpha \in G$ on a point $x \in \M$. \\
$\theta \colon G \times \M \to \M$ & Group action map, with $\theta(\alpha,x)=\alpha\cdot x$. \\
$[x] = G\cdot x = \{ \alpha \cdot x \mid \alpha \in G \}$ & Orbit of $x$ under the action of $G$; equivalently, the quotient point represented by $x$. \\
$\N=\M/G$ & Quotient manifold, or orbit space, obtained by identifying points in the same $G$-orbit. \\
$q = \dim(\N) = d-p$ & Dimension of the quotient manifold. \\
$\pi \colon \M \to \N$ & Quotient map, $\pi(x)=[x]$. \\
$dV_\M,\ dV_\N,\ dV_{G \cdot x}$ & Riemannian volume measures on $\M$, $\N$ and the orbit $G \cdot x$. \\
$d\eta$ & Haar measure on $G$, normalized so that $\int_G 1\,d\eta=1$. \\
$\delta,\overline{\delta}$ & Orbit-volume density on $\M$ (defined by $dV_{G\cdot x}=\delta\,d\eta$) and its induced function on $\N$. \\

\midrule
\multicolumn{2}{@{}l}{\textit{Functions and operators}}\\[0.2em]
$f,\overline{f}$ & Smooth $G$-invariant function on $\M$ and its induced function on $\N$. \\
$h \colon \N \to \RR$ & Generic smooth test function on $\N$. \\
$\rho,\tilde{\rho}$ & Sampling density on $\M$ for non-uniform data, and its orbit average on $\N$. \\
$\phi \colon \RR^D \to \RR^E$, $\bar{\phi} \colon \N \to \mathrm{im}\phi$ & $G$-invariant feature map and its induced map on $\N$. \\
$p_\phi$ & Density of the pushforward measure $\phi_\ast(dV_\M)$ on $\mathrm{im}\phi$. \\
$\Delta_\M,\Delta_\N,\Delta_{\mathrm{im}\phi}$ & Laplace--Beltrami operators on $\M$, $\N$, and $\mathrm{im}\phi$. \\
$\nabla_\M,\nabla_\N,\nabla_{\mathrm{im}\phi}$ & Riemannian gradients on $\M$, $\N$, and $\mathrm{im}\phi$. \\
$P(\Delta_\M)$ & Projection of the Laplace--Beltrami operator on $\M$ to an operator on $\N$. \\
$\D$ & Limiting second-order differential operator on $\N$. \\
$\D_\rho$ & Non-uniform-sampling version of the limiting operator $\D$. \\
$\pi^\ast,\bar{\phi}^\ast,\phi_\ast$ & Pullbacks by $\pi$ and $\bar{\phi}$, and pushforward by $\phi$. \\
\end{longtable}
\normalsize

\end{document}

%% file: torus_orbit_intro_figure_tikz.tex
\begin{tikzpicture}[x=1cm,y=1cm,line cap=round,line join=round]

\draw[gray!15,line width=0.45pt] (0,0) rectangle (4.8,4.8);
\draw[CRed,edge id] (0,4.8) -- (4.8,4.8);
\draw[CRed,edge id] (0,0) -- (4.8,0);
\draw[CBlue,edge id] (0,0) -- (0,4.8);
\draw[CBlue,edge id] (4.8,0) -- (4.8,4.8);

\node[CRed,font=\small] at (2.4,5.08) {$a$};
\node[CRed,font=\small] at (2.4,-0.28) {$a$};
\node[CBlue,font=\small] at (-0.28,2.4) {$b$};
\node[CBlue,font=\small] at (5.08,2.4) {$b$};
\node[CDark,font=\small] at (2.4,-0.68) {$\M=\mathbb{T}^2$};

\draw[COrange,orbit] (0.45,3.25) -- (4.35,3.25);
\draw[CGreen,orbit] (0.45,1.45) -- (4.35,1.45);

\node[pt,fill=COrange!85!black] at (1.05,3.25) {};
\node[pt,fill=COrange!85!black] at (3.55,3.25) {};
\node[pt,fill=CGreen!60!black] at (2.40,1.45) {};

\draw[->,COrange!75!black,line width=0.75pt,densely dashed]
  (1.22,3.25) -- (3.38,3.25);

\node[font=\small,anchor=south] at (1.05,3.50) {$x$};
\node[font=\small,anchor=south] at (3.55,3.50) {$x'=g\cdot x$};
\node[font=\small,anchor=north] at (2.40,1.20) {$y$};

\draw[map] (5.65,2.40) -- (7.28,2.40);
\node[font=\small,anchor=south,CDark] at (6.46,2.53) {$\pi$};

\def\xq{8.55}
\def\ybot{0.00}
\def\ytop{4.80}
\coordinate (Qx) at (\xq,3.25);
\coordinate (Qy) at (\xq,1.45);

\draw[CBlue,line width=1.15pt] (\xq,\ybot) -- (\xq,\ytop);
\draw[CBlue,line width=0.9pt] (\xq-0.13,\ybot) -- (\xq+0.13,\ybot);
\draw[CBlue,line width=0.9pt] (\xq-0.13,\ytop) -- (\xq+0.13,\ytop);
\node[CDark,font=\small] at (\xq,-0.68) {$\N=\M/G\cong\mathbb{S}^1$};
\node[CDark,font=\scriptsize] at (\xq,-1.02) {endpoints identified};

\node[pt,fill=COrange!85!black] at (Qx) {};
\node[pt,fill=CGreen!60!black] at (Qy) {};
\node[font=\small,anchor=west] at (\xq+0.28,3.25) {$[x]=[x']$};
\node[font=\small,anchor=west] at (\xq+0.28,1.45) {$[y]$};

\draw[dist] (\xq-0.42,1.45) -- (\xq-0.42,3.25);
\node[CPink,font=\small,rotate=90,anchor=south] at (\xq-0.46,2.35)
  {$d_{\N}([x],[y])$};

\end{tikzpicture}